\documentclass{article}
\usepackage{bm}
\usepackage{wrapfig}

\newcommand{\paren}[1]{\left(#1\right)}

\newcommand{\norm}[1]{\left\|#1\right\|}

\newcommand{\R}{\mathbb{R}}

\newcommand{\Z}{\mathbb{Z}}

\def\bfA{{\mathbf{A}}}

\def\bfD{{\mathbf{D}}}

\def\bfH{{\mathbf{H}}}
\def\bfI{{\mathbf{I}}}

\def\bfL{{\mathbf{L}}}
\def\bfM{{\mathbf{M}}}

\def\bfQ{{\mathbf{Q}}}

\def\bfU{{\mathbf{U}}}
\def\bfV{{\mathbf{V}}}

\def\bfa{{\mathbf{a}}}
\def\bfb{{\mathbf{b}}}

\def\bfu{{\mathbf{u}}}
\def\bfv{{\mathbf{v}}}

\def\bfx{{\mathbf{x}}}

\def\bfz{{\mathbf{z}}}

\newcommand{\pca}[1]{\texttt{PCA}\paren{#1}}

\makeatletter
\newcommand{\newreptheorem}[2]{%
\newenvironment{rep#1}[1]{%
 \def\rep@title{#2 \ref{##1}}%
 \begin{rep@theorem}}%
 {\end{rep@theorem}}}
\makeatother

\newreptheorem{theorem}{Theorem}
\newreptheorem{lemma}{Lemma}

\usepackage{microtype}
\usepackage{graphicx}
\usepackage{subfigure}
\usepackage{booktabs} 

\usepackage{hyperref}

\usepackage[accepted]{icml2024}

\usepackage{amsmath}
\usepackage{amssymb}
\usepackage{mathtools}
\usepackage{amsthm}
\usepackage{enumitem}
\usepackage[capitalize,noabbrev]{cleveref}

\usepackage[capitalize,noabbrev]{cleveref}

\theoremstyle{plain}
\newtheorem{theorem}{Theorem}[section]

\newtheorem{lemma}[theorem]{Lemma}
\newtheorem{corollary}[theorem]{Corollary}
\theoremstyle{definition}

\newtheorem{assumption}[theorem]{Assumption}
\theoremstyle{remark}

\usepackage[textsize=tiny]{todonotes}

\icmltitlerunning{On the Error-Propagation of Inexact Deflation for PCA}

\begin{document}

\twocolumn[
\icmltitle{On the Error-Propagation of Inexact Hotelling's Deflation \\ for Principal Component Analysis}



\icmlsetsymbol{equal}{*}

\begin{icmlauthorlist}
\icmlauthor{Fangshuo Liao}{ricecs}
\icmlauthor{Junhyung Lyle Kim}{ricecs}
\icmlauthor{Cruz Barnum}{uiuc}
\icmlauthor{Anastasios Kyrillidis}{ricecs}
\end{icmlauthorlist}

\icmlaffiliation{ricecs}{Department of Computer Science, Rice University, Houston, U.S.A.}
\icmlaffiliation{uiuc}{Department of Computer Science, UIUC}

\icmlcorrespondingauthor{Fangshuo Liao}{Fangshuo.Liao@rice.edu}

\icmlkeywords{Principal Component Analysis; Inexact Deflation; Power Iteration.}

\vskip 0.3in
]


\printAffiliationsAndNotice{\icmlEqualContribution} 

\begin{abstract}
Principal Component Analysis (PCA) aims to find subspaces spanned by the so-called \textit{principal components} that best represent the variance in the dataset. 
The deflation method is a popular meta-algorithm that sequentially finds individual principal components, starting from the most important ones and working towards the less important ones.
However, as deflation proceeds, numerical errors from the imprecise estimation of principal components propagate due to its sequential nature. This paper mathematically characterizes the error propagation of the inexact Hotelling's deflation method.
We consider two scenarios: $i)$ when the sub-routine for finding the leading eigenvector is abstract and can represent various algorithms; 
and $ii)$ when power iteration is used as the sub-routine. 
In the latter case, the additional directional information from power iteration allows us to obtain a tighter error bound than the sub-routine agnostic case.
For both scenarios, we explicitly characterize how the errors progress and affect subsequent principal component estimations.
\end{abstract}

\section{Introduction}
Principal Component Analysis (PCA) \citep{pearson1901liii, hotelling1933analysis} is a fundamental tool for data analysis with applications that range from statistics to machine learning and can be used for dimensionality reduction, classification, and clustering, to name a few \citep{majumdar2009image,wang2013sparse, d2007direct,jiang2011anomaly,zou2006sparse}.
Since its proposition \citet{pearson1901liii}, many algorithms for PCA have been devised, even recently \citep{allenzhu2017efficient}, to efficiently find approximations of the first principal component, which is a vector that captures the most variance in the data. Mathematically, single-component PCA can be described as:
\begin{equation}
    \vspace{-0.2cm}
    \bfu_1^* = \texttt{PCA}(\boldsymbol{\Sigma}) = 
    \underset{\substack{ \mathbf{v} \in  \mathbb{R}^d: \|\mathbf{v}\|_2 = 1 \\} }{\arg\max}
    \mathbf{v}^{\top} \boldsymbol{\Sigma} \mathbf{v},
    \label{eq:principal-component}
    \vspace{-0.1cm}
\end{equation}
where $\bm{\Sigma} = \tfrac{1}{n} \cdot \mathbf{Y}\mathbf{Y}^{\top} \in \mathbb{R}^{d \times d}$ is the empirical covariance matrix given a set of $n$ centered $d$-dimensional data-points~${\mathbf{Y} \in \mathbb{R}^{d \times n}}$. 
\begin{figure}
    \centering
    \includegraphics[width=0.75\linewidth]{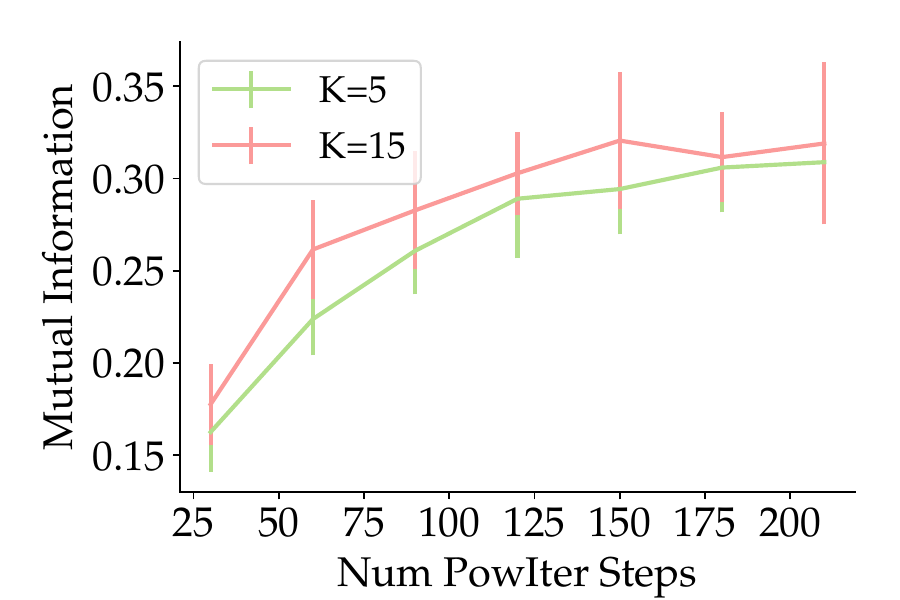}
    \caption{\textit{Spectral clustering of MNIST dataset using inexact deflation method in Algorithm~\ref{alg:main-alg}}. As the number of power iteration steps increases ($x$-axis), clustering performance, measured by the mutual information metric, also improves. A similar pattern is observed for recovering different numbers of eigenvectors. 
    }
    \label{fig:dfl_spec_clst}
    \vspace{-0.5cm}
\end{figure}

It is sometimes helpful to know not only the first but also the top-$K$ principal components, especially in the high-dimensional regime. 
This is becoming more relevant in the modern machine learning era, where the number of features reaches as large as billions \citep{fedus2022switch, touvron2023llama}. 
Mathematically, the multi-component PCA ($K \geq 1$) can be formulated as the following constrained optimization problem: 
\begin{equation}
    \vspace{-0.2cm}
    \bfU^* = \underset{\bfV\in\{\bfQ_{:,:K}:\;\bfQ\in\text{SO}(d)\}}{\arg\max}
    \left\langle \boldsymbol{\Sigma} \mathbf{V}, \mathbf{V} \right \rangle,
    \label{eq:multi-principal-component}
    \vspace{-0.1cm}
\end{equation}
where $\text{SO}(d)$ denotes the group of rotations about a fixed point in $d$-dimensional Euclidean space. 
Since $\bfU^*$ contains the first $K$ columns of an orthogonal matrix, each principal component $\bfU_{:,k}$ must be a unit vector, and the pairwise principal components, say $\bfU^*_{:,k_1}, \bfU^*_{:,k_2}$ for $k_1\neq k_2$, are orthogonal. 
A popular method for solving \eqref{eq:multi-principal-component} is the \textit{deflation method} \citep{hotelling1933analysis} \textit{via sequentially solving} \eqref{eq:principal-component}, which is the focus of our work. 

Deflation proceeds as follows. 
After approximating the top component $\mathbf{u}_1^*$, the matrix $\bm{\Sigma}$ is further processed to exist on the subspace that is orthogonal to the subspace which is spanned by the first component. This process is repeated by applying single-component PCA in \eqref{eq:principal-component} on the deflated matrix, which leads to an approximation of the second component $\mathbf{u}_2^*$, and so on, as below:
\begin{equation}
    \label{eq:deflation}
    \begin{gathered}
    \bm{\Sigma}_1 = \bm{\Sigma};\quad \bfv_k = \pca{\bm{\Sigma}_k, t};\\
    \bm{\Sigma}_{k+1} = \bm{\Sigma}_k - \bfv_k\bfv_k^\top\bm{\Sigma}_k\bfv_k\bfv_k^\top,
    \end{gathered}
\end{equation}
where $\pca{\bm{\Sigma}_k, \cdot}$ returns a normalized approximation of the top eigenvector of the deflated matrix $\bm{\Sigma}_k$. 
One estimates the subsequent principal components by running the same single-component PCA algorithm repetitively.

In this work, we focus on the more realistic scenario where the sub-routine, $\pca{\bm{\Sigma}_k, \cdot}$, incurs numerical errors due to limited computation budget and finite precision: even solving  \eqref{eq:principal-component} for a single component, our estimate is only an approximation to the true principal component. 
The overall procedure with inexact $\pca{\bm{\Sigma}_k, \cdot}$ sub-routine is described in Algorithm~\ref{alg:main-alg}. 
Line 4 performs the sub-routine algorithm for $t$ iterations, denoted by $\texttt{PCA}(\bm{\Sigma}_k,t)$, which approximately computes the top eigenvector of $\bm{\Sigma}_k$. 
In line 5, $\bm{\Sigma}_k$ is deflated using $\bfv_k$; hence, the numerical error from approximately solving $\bfv_k$ in line 4 affects the quality of the deflated matrix $\bm{\Sigma}_{k+1}$ in the next iteration. 

\begin{algorithm}[tb]
   \caption{Deflation with inexact \texttt{PCA}}
   \label{alg:main-alg}
\begin{algorithmic}[1]
        \REQUIRE $\bm{\Sigma}\in\R^{d\times d}$; \# of eigenvectors $K$; sub-routine for top eigenvector $\texttt{PCA}(\cdot,\cdot)$; \# of iterations $t$.
        \ENSURE Approximate eigenvectors $\bfV = \left\{\bfv_k\right\}_{k=1}^K$.\\
        \hrulefill\\
        \STATE $\bfV \leftarrow \emptyset$
        \STATE $\bm{\Sigma}_1 \leftarrow \bm{\Sigma}$
        \FOR{$k =1,\dots, K$}
            \STATE $\bfv_k \leftarrow \texttt{PCA}\paren{\bm{\Sigma}_k, t}$
            \STATE $\bm{\Sigma}_{k+1} \leftarrow \bm{\Sigma}_k - \bfv_k\bfv_k^\top\bm{\Sigma}_k\bfv_k\bfv_k^\top$
            \STATE $\bfV \leftarrow \bfV \cup\left\{\bfv_k\right\}$
        \ENDFOR
        \STATE\textbf{return} $\bfV$
\end{algorithmic} 
\end{algorithm}

To motivate why studying the accumulation of errors is critical, we apply spectral clustering \cite{pothen1990partitioning} on the MNIST dataset and measure its performance using the mutual information metric\footnote{Higher mutual information indicates more accurate clustering. For more details about the experiment, please see Appendix~\ref{sec:exp_details}.}. 
Spectral clustering involves computing the top-$K$ eigenvectors from a similarity matrix and uses the corresponding entries in the $K$ eigenvectors as features for the following clustering procedure. 
In our experiments, eigenvectors are computed using the deflation in Algorithm~\ref{alg:main-alg} with different precision levels of the $\texttt{PCA}$ sub-routine, indicated by different power iteration steps. 
We can observe from Figure~\ref{fig:dfl_spec_clst} that, as the number of power iteration steps increases, the mutual information increases, showing a higher clustering accuracy. 
This shows that the errors incurred from inexactly solving each top eigenvector in the deflation process indeed have a non-negligible impact on the downstream machine learning tasks.

This interdependent accumulation of errors coming from approximately solving the sub-routine --which results in the less accurate deflation matrix, and further affects the quality of the approximation of the top eigenvector in the next iteration-- is precisely what we characterize in this work.
To the best of our knowledge, \textit{this is the first work that analyzes the inexact setting, although this procedure is vastly used in practice.} 
Our contributions can be summarized as follows: \vspace{-0.2cm}
\begin{itemize}[leftmargin=*]
    \item In Theorem~\ref{thm:main_theorem_1}, we characterize how the errors from approximately solving the sub-routine (\texttt{PCA}) propagate into the overall error of the deflation procedure. Informally, the error between the approximate principal component and the actual one needs to be exponentially small to control the accuracy for the top-$K$ eigenvectors for an increasing $K$, where the base of the exponential growth depends on the eigengap of the target matrix. \vspace{-0.15cm}
    \item In Theorem~\ref{thm:power-iteration}, we consider the case where $\texttt{PCA}$ is fixed to the power iteration \citep{muntz1913solution}. By leveraging the directional information of the numerical error vector produced by each power iteration, we improve the error bound in Theorem~\ref{thm:main_theorem_1} (which is agnostic to the $\texttt{PCA}$ sub-routine) to exponential growth with constant base. \vspace{-0.15cm}
    \item Along with the proof, we also present empirical results to justify the core ideas in our theoretical analysis. \vspace{-0.15cm}
\end{itemize}

\textbf{Related work.}
Most works focus on the convergence proof of numerically solving the top eigenvector of matrices \citep{1950AnIM,jain2016streaming,xu2018accelerated}. 
However, these results only concern the case of the first eigenvector, and their result does not readily apply to numerical algorithms, such as deflation, that solve for multiple top eigenvectors.

Iterative methods for multi-component PCA include \textit{Schur complement deflation} \citep{zhang2006schur}, \textit{orthogonalized deflation} \citep{sriperumbudur2007sparse}, and \textit{projection methods} \citep{saad1988projection}. 
Among them, Hotelling's deflation, our focus, was reported to have the highest accuracy in practice \citep{danisman2014comparison}. 
For the comparison of these methods, we refer the readers to \citet[Section~2]{mackey2008deflation}, which extends some of the above deflation methods to the sparse PCA setting \citep{zou2006sparse, johnstone2009sparse}, which is out of scope of our work. Importantly, all the works mentioned above assume the $\texttt{PCA}$ step in \eqref{eq:deflation} is solved exactly, i.e., $t=\infty$. Works that analyze the resulting error when solving multiple top eigenvectors include \citet{allenzhu2017efficient,gemp2021eigengame}. However, the theoretical analysis in \citet{allenzhu2017efficient} and \citet{gemp2021eigengame} only applies to specific algorithms such as EigenGame and Oja++. 
At the same time, our work focuses on deflation, a widely used algorithmic framework for solving multiple eigenvectors.

Subspace methods such as orthogonal iterations \citep{golub2013matrix} are also widely used for discovering multiple eigenvectors. However, the deflation serves a different purpose than subspace methods. In particular, the sequential discovery of the eigenvectors in the deflation methods applies to scenarios where users have a dynamic necessity of the eigenvectors and require flexibility, especially in modern ML settings \citep{Fawzi2022,coulaud2013deflation}. For instance, compared with orthogonal iteration, which outputs all $k$ eigenvectors at the same time, the deflation method will be more suitable for scenarios where the top principal component is known but more eigenvectors need to be computed depending on the need, and scenarios where eigenvectors are output in a consecutive way to facilitate downstream tasks.

In terms of the analysis technique, our work utilizes Weyl's inequality \citep{weyl1912asymptotische} and Davis-Kahan $\sin\Theta$ Theorem \citep{davis1970rotation} to characterize the differences between eigenvalues of two matrices. In our fine-grained analysis in Section~\ref{sec:main_result_pi}, we also utilize the Neumann expansion introduced by \citet{eldridge2017unperturbed} and characterize the subspace that the top eigenvector of a perturbed matrix lies in with a similar technique to \citet{chen2020asymmetry}. However, the problem they study is entirely orthogonal to ours.

\section{Problem Setup and Notations}
\textbf{Notation.} For a vector $\bfa$, we use $\norm{\bfa}_2$ to denote its $\ell_2$-norm. For a matrix $\bfA$, $\norm{\bfA}_2$ to denote its spectral norm and $\norm{\bfA}_F$ its Frobenius norm. $\lambda_j\paren{\bfA}$ denotes its $j$-th eigenvalue.

\textbf{Problem setup.}
Given a symmetric matrix $\bm{\Sigma}\in\R^{d\times d}$ with eigenvectors $\bfu_1^*,\dots\bfu_d^*$ and eigenvalues $\lambda_1^*,\dots,\lambda_d^*$, sorted in descending order (c.f., Assumption~\ref{assu:main-assumption}), we want to estimate the top $K \leq d$ eigenvectors using the deflation method. 
Formally, let $\texttt{PCA}$ be a sub-routine that returns an approximation of the top eigenvector of a given matrix. 
The deflation method repeatedly computes approximation $\bfv_k$ of the principal eigenvector of the current deflation matrix $\bm{\Sigma}_k$. 
It forms a new deflation matrix $\bm{\Sigma}_{k+1}$ by subtracting the component of $\bfv_k$ from $\bm{\Sigma}_k$. 
This process is made formal in Algorithm~\ref{alg:main-alg} and is mathematically defined in (\ref{eq:deflation}). 
If the sub-routine $\texttt{PCA}$ is solved precisely, then the output satisfies $\bfv_k = \bfu_k^*$. 
In this ideal scenario, the sequence of deflation matrices has the following form:
\begin{equation}
    \label{eq:true-deflation-mats}
     \bm{\Sigma}^*_{k+1} = \bm{\Sigma}^*_k - \bfu_k^*\bfu_k^{*\top}\bm{\Sigma}^*_k\bfu_k^*\bfu_k^{*\top};\quad\bm{\Sigma}^*_1 = \bm{\Sigma}.
\end{equation}
We term the scenario in (\ref{eq:true-deflation-mats}) as \textit{ideal deflation} and the matrices $\left\{\bm{\Sigma}_k^*\right\}_{k=1}^d$ as \textit{``ground-truth'' deflation matrices}. 
The ideal deflation has the nice property that $\bfu_k^*$ is the top eigenvector of $\bm{\Sigma}_k^*$. 
Therefore, solving for the top-$K$ eigenvectors can be exactly reduced to repeatedly solving for the top-$1$ eigenvector of the ``ground-truth'' deflation matrices. 

Yet, solving for the top eigenvector exactly is challenging, which makes the ideal deflation almost impossible in practice. 
Thus, we are interested in a more general case, where $\texttt{PCA}$ computes the top-eigenvector \textit{inexactly}, and results in numerical error $\bm{\delta}_k$, as defined below:
\begin{equation} 
    \label{eq:def-delta-v}
    \begin{gathered}
        \bm{\delta}_k := \pca{\bm{\Sigma}_k, t} - \bfu_{k} = \bfv_k - \bfu_k;~~ \norm{\bm{\delta}_k}_2 > 0.
    \end{gathered}
\end{equation}
Here, $\bfu_k$ is the top eigenvector of $\bm{\Sigma}_k$; recall that $\bm{\Sigma}_k$ in (\ref{eq:deflation}) is constructed recursively using $\bfv_k$. 
When $\bfv_k$ is solved inexactly, we cannot guarantee that $\bm{\Sigma}_k = \bm{\Sigma}_k^*$. 
Consequently, it is almost always the case that the top eigenvector of $\bm{\Sigma}_k$ does not equal to $\bfu_k^*$, namely $\norm{\bfu_k - \bfu_k^*}_2 > 0$. 
However, since $\texttt{PCA}$ only knows $\bm{\Sigma}_k$, its output $\bfv_k$ converges to $\bfu_k$ instead of $\bfu_k^*$ when $t$ is large. 
This difference between the top eigenvector of the matrices $\bm{\Sigma}_k$ and $\bm{\Sigma}_k^*$ further complicates the analysis of the deflation process in (\ref{eq:deflation}), since $\bfv_k$ returned by $\texttt{PCA}\paren{\bm{\Sigma}_k,t}$ approximates $\bfu_k$ instead of $\bfu_k^*$, which further builds up the difference between $\bfv_k$ and $\bfu_k^*$ and thus the subsequent $\norm{\bm{\Sigma}_{k+1} - \bm{\Sigma}_{k+1}^*}_F$. 

Our work aims to characterize this complicated error propagation through the deflation steps. 
To build up our analysis, we make the following assumption on the matrix $\bm{\Sigma}$.
\begin{assumption} \label{assu:main-assumption}
    The matrix $\bm{\Sigma}\in\R^{d\times d}$ is a real symmetric matrix with eigenvalues and eigenvectors $\{\lambda_i^*,\bfu_i^*\}_{i=1}^d$, satisfying $1 = \lambda_1^* > \lambda_2^* > \dots > \lambda_d^* > 0$.
\end{assumption}
Assumption~\ref{assu:main-assumption} states the real symmetric matrix $\bm{\Sigma}$ has strictly decaying eigenvalues. 
Observing such phenomenon is quite common in practice; see \citet[Figure 1]{papailiopoulos2014provable}, \citet[Figure 1]{papyan2020traces}, and \citet[Figure 1]{goujaud2022super} for instance. 
Further, $\lambda_1^*=1$ is not strictly required but assumed without loss of generality. 

Assumption~\ref{assu:main-assumption} guarantees that $\bm{\Sigma}$ is positive definite. 
By construction, all $\bm{\Sigma}_k^*$'s are guaranteed to be positive semi-definite. 
To enforce a similar condition on $\bm{\Sigma}_k$'s, we need to characterize the difference between the eigenvalues of $\bm{\Sigma}_k^*$ and $\bm{\Sigma}_k$: for $\lambda_k$ being the top eigenvalue of $\bm{\Sigma}_k$, we need to control $\big|\lambda_k^*-\lambda_k\big|$. 
This is possible by  Weyl's inequality:
\begin{lemma}[Weyl's Inequality  \citep{weyl1912asymptotische}]
    Let $\bfM,\bfM^*\in\R^{d\times d}$ be real symmetric matrices. Let $\sigma_j,\sigma_j^*$ be the $j$-th eigenvalue of $\bfM$ and $\bfM^*$, respectively. Then:
    \[
        \left|\sigma_j^* - \sigma_j\right|\leq \norm{\bfM^* - \bfM}_2.
    \]
\end{lemma}
A direct consequence of Weyl's inequality to our scenario is that $\big|\lambda_k^*-\lambda_k\big| \leq \norm{\bm{\Sigma}_k - \bm{\Sigma}_k^*}_2$. 

Another tool that we will utilize in the following sections of our paper is the Davis-Kahan $\sin\Theta$ Theorem:
\begin{lemma}[$\sin\Theta$ Theorem \cite{davis1970rotation}]
    \label{lem:davis-kahan}
    Let $\bfM^*\in\R^{d\times d}$ and let $\bfM = \bfM^* + \bfH$. Let $\bfa_1^*$ and $\bfa_1$ be the top eigenvectors of $\bfM^*$ and $\bfM$, respectively. Then we have:
    \[
        \sin\angle\left\{\bfa_1^*,\bfa_1\right\} \leq \frac{\norm{\bfH}_2}{\min_{j\neq k}\left|\sigma_k^* - \sigma_j\right|}.
    \]
\end{lemma}

\section{Sub-routine Agnostic Error Propagation}
\label{sec:main_result_general}
In this section, we aim to provide a characterization of the error propagation in the deflation methods in (\ref{eq:deflation}) that is agnostic to the detail of the sub-routine $\texttt{PCA}$. 
In particular, the only information known about $\texttt{PCA}$ is the magnitude of the error of $\texttt{PCA}$, namely $\norm{\bm{\delta}_k}_2 = \norm{\bfv_k -\bfu_k}_2 = \norm{\pca{\bm{\Sigma}_k, t} - \bfu_{k}}_2$. 
\begin{theorem}
    \label{thm:main_theorem_1}
    Consider the scenario of looking for the top-$K$ eigenvectors of ~$\bm{\Sigma} \in\R^{d\times d}$, and 
    suppose $\bm{\Sigma}$ satisfies Assumption~\ref{assu:main-assumption}. 
    Let $\mathcal{T}_j = \lambda_j^* - \lambda_{j+1}^*$, for $j\in[d-1]$, and $\mathcal{T}_d = \lambda_d^*$. 
    Also, let  $\mathcal{T}_{K,\min}=\min_{k\in[K]}\mathcal{T}_k$, and let $\bm{\delta}_k$ and $\bfv_k$ be defined in \eqref{eq:def-delta-v}.  
    If $\norm{\bm{\delta}_k}_2$'s are small enough such that
    \begin{equation}
        \label{eq:thm1_requirement}  
        \sum_{k'=1}^{k-1}\lambda_{k'}^*\norm{\bm{\delta}_{k'}}_2\prod_{j=k'+1}^{k-1}\paren{3 + \frac{2\lambda_j^*}{\mathcal{T}_j}} \leq \frac{1}{20}\mathcal{T}_{K,\min}.
    \end{equation}
    then the output of Algorithm~\ref{alg:main-alg} satisfies for all $k\in[K]$:
    \begin{align}
        \label{eq:thm1_conclusion}
        \norm{\bfv_k - \bfu_k^*}_2 &\leq \norm{\bfu_k - \bfu_k^*}_2 + \norm{\bm{\delta}_k}_2 \nonumber \\ &\leq 5\sum_{k'=1}^{k}\frac{\lambda_{k'}^*}{\lambda_k^*}\norm{\bm{\delta}_{k'}}_2\prod_{j=k'+1}^{k}\paren{3 + \frac{2\lambda_j^*}{\mathcal{T}_j}}.        
    \end{align}
    \vspace{-0.5cm}
\end{theorem}
\textbf{Remark 1.} Theorem~\ref{thm:main_theorem_1} characterizes how the errors from approximately solving the sub-routine (\texttt{PCA}) propagate into the overall error of the deflation procedure. 
The condition in (\ref{eq:thm1_requirement}) is to make sure that the difference between the ``ground-truth'' deflated matrices in \eqref{eq:true-deflation-mats} and the empirical ones in \eqref{eq:deflation} is controlled and is sufficiently small, i.e., $\norm{\bm{\Sigma}_k - \bm{\Sigma}^*_k}_F \leq \tfrac{1}{4}\mathcal{T}_{k,\min}$ for all $k\in[K]$. As a simplification, this condition can be guaranteed as long as
\[
    \norm{\bm{\delta}_k}_2 \leq \tfrac{\mathcal{T}_{K,\min}}{20K}\prod_{j=k+1}^{K-1}\paren{3 + \tfrac{2\lambda_j^*}{\mathcal{T}_j}}^{-1}.\vspace{-0.2cm}
\]
\textbf{Remark 2.} Notice that the error upper bound in (\ref{eq:thm1_conclusion}) takes the form of a summation over components that depend on the error of the sub-routine $\left\{\norm{\bm{\delta}_{k'}}_2\right\}_{k'=1}^k$. In particular, not only do the number of summands grow as $k$ grows, but each summand also has a multiplicative factor of $\prod_{j=k'+1}^{k}\paren{3 + \tfrac{2\lambda_j^*}{\mathcal{T}_j}}$ which grows as $k$ becomes larger. When the eigen spectrum of $\bm{\Sigma}^*$ decays slowly, this multiplicative factor can grow near factorially. For instance, in the case of a power-law decay spectrum where $\lambda_j = \tfrac{1}{j}$, the multiplicative factor becomes $\prod_{j=k'+1}^{k}\paren{3 + 2j}$, since $\mathcal{T}_j = \tfrac{1}{j(j+1)}$. In this case, the errors $\left\{\norm{\bm{\delta}_{k'}}_2\right\}_{k'=1}^k$ need to be nearly factorially small as we attempt to solve for more eigenvectors. To be more specific, based on Theorem~\ref{thm:main_theorem_1}, to guarantee that $\norm{\bfv_k - \bfu_k^*}_2\leq \varepsilon$, it is \textit{necessary} to have:
\[
    \norm{\bm{\delta}_{k'}}\leq \frac{\varepsilon\lambda_k^*}{5\lambda_{k'}^*}\prod_{j=k'+1}^{k}\paren{3 + \frac{2\lambda_j^*}{\mathcal{T}_j}}^{-1}.
\]
The corollary below states a \textit{sufficient} condition to guarantee that $\norm{\bfv_k - \bfu_k^*}_2\leq \varepsilon$ for all $k\in[K]$.
\begin{corollary}
    \label{cor:main_cor_delta1}
    Consider the scenario of solving the top-$K$ eigenvectors of ~$\bm{\Sigma} \in\R^{d\times d}$, and 
    suppose $\bm{\Sigma}$ satisfies Assumption~\ref{assu:main-assumption}. 
    Let $\mathcal{T}_j = \lambda_j^* - \lambda_{j+1}^*$ for $j\in[d-1]$, and $\mathcal{T}_d = \lambda_d^*$.
    Also let  $\mathcal{T}_{K,\min}=\min_{k\in[K]}\mathcal{T}_k$.
    Finally, let $\bm{\delta}_i$ and $\bfv_i$ be defined as in \eqref{eq:def-delta-v}.  
    If for all $k\in[K]$ it holds that:
    \begin{equation}
        \label{eq:cor1_requirement}
        \norm{\bm{\delta}_k}_2\leq \tfrac{\min\left\{\varepsilon\lambda_K^*, \mathcal{T}_{K,\min}\right\}}{20K}\prod_{j=k+1}^{K}\paren{3 + \frac{2\lambda_j^*}{\mathcal{T}_j}}^{-1},
    \end{equation}
    then the output of Algorithm~\ref{alg:main-alg} satisfies $\norm{\bfv_k - \bfu_k^*}_2\leq \varepsilon$ for all $k\in[K]$.
\end{corollary}
\textbf{Remark 3.}
Corollary~\ref{cor:main_cor_delta1} characterizes how accurately the sub-routine (\texttt{PCA}) has to be solved to achieve $\norm{\bfv_k - \bfu_k^*}_2\leq \varepsilon$, $\forall k\in[K]$, i.e., the desired accuracy of the overall deflation procedure. 

Below, we provide lower bounds on the number of sub-routine iterations to satisfy \eqref{eq:cor1_requirement}, assuming a linearly-converging sub-routine with convergence rate $\alpha_k$ exists. In Section~\ref{sec:main_result_pi}, we improve this rate by specifically setting the sub-routine to be the \textit{power iteration} \citep{muntz1913solution}.
\begin{corollary}
    \label{cor:main_cor_t}
    Under the same assumptions as Corollary~\ref{cor:main_cor_delta1}, we further assume a sub-routine exists for computing the top eigenvalue with linear convergence rate $\alpha_k \in (0, 1)$. 
    If $\norm{\bm{\delta}_k}_2 \leq c_0 \cdot\alpha_k^t$ for some constant $c_0 > 0$ for all $k\in[K]$, where $t$ is the number of steps in the sub-routine satisfying:
    \begin{equation}
        \label{eq:main_cor_t}
        t \geq \Omega\paren{\frac{\log \frac{c_0K}{\min\left\{\varepsilon\lambda_k^*, \mathcal{T}_{K,\min}\right\}}+ \sum_{j=k+1}^K\log\paren{\frac{\lambda_j^*}{\mathcal{T}_j} + 1}}{\log\alpha_k^{-1}}},
    \end{equation}
    then the output of Algorithm~\ref{alg:main-alg} satisfies $\norm{\bfv_k - \bfu_k^*}_2\leq \varepsilon$ for all $k\in[K]$.
\end{corollary}
\textbf{Remark 4.}
Corollary~\ref{cor:main_cor_t} is constructed by further assuming that $\norm{\bm{\delta}_k}_2\leq c_0\cdot\alpha_k^t$ based on Corollary~\ref{cor:main_cor_delta1} and derive the lower bound on $t$. The linear convergence property of solving the top eigenvector is shared among multiple existing algorithms \citep{desa2017accelerated, GoluVanl96}. In particular, when $\texttt{PCA}$ is chosen to be the power iteration, we can show that in our setting $\norm{\bm{\delta}_k}_2\leq c_0\paren{\frac{\lambda_{2}\paren{\bm{\Sigma}_k}}{\lambda_k}}^t$. In the case of a power-law decay spectrum, when we treat $c_0$ as constants, the bound in (\ref{eq:main_cor_t}) can be simplified to
\[
    t\geq \Omega\paren{\log \frac{K}{\varepsilon} + \sum_{j=k+1}^K\log\paren{j+1}}.\vspace{-0.2cm}
\]
Applying a slightly loose bound over the summation, we can see that based on the analysis of Theorem~\ref{thm:main_theorem_1}, one needs $t \geq \Omega\paren{\log \frac{K}{\varepsilon} + K\log K}$ to achieve $\varepsilon$-accuracy. Section~\ref{sec:main_result_pi} compares this rate with the lower bound we obtain when $\texttt{PCA}$ is fixed to power iteration.

\subsection{Proof Overview}
\label{sec:proof_overview_thm1}
The proof involves bounding $\norm{\bfu_{k}- \bfu_k^*}_2$ and $\norm{\bm{\Sigma}_k - \bm{\Sigma}^*_k}_F$. 
The former quantifies the difference between the true $k$-th eigenvector of $\bm{\Sigma}$ and the leading eigenvector of the $k$-th deflation step $\bm{\Sigma}_k$. 
The latter quantifies the difference between the ``ground-truth'' deflation matrices defined in \eqref{eq:true-deflation-mats} and the ones in practice. Recall that $\bfu_k^*$ is the top eigenvector of $\bm{\Sigma}_k^*$ (c.f., Lemma~\ref{lem:true_deflate_eigen_decomp}). Let $\texttt{nev}_1\paren{\bm{\Sigma}_k}$ denote a normalized top eigenvector of $\bm{\Sigma}_k$. Since both $\texttt{nev}_1\paren{\bm{\Sigma}_k}$ and $-\texttt{nev}_1\paren{\bm{\Sigma}_k}$ are unit-norm eigenvectors of $\bm{\Sigma}_k$, we will choose $\bfu_k$ to be the one such that $\bfu_k^\top\bfu_k^*\geq 0$:
\begin{equation}
    \label{eq:uk_defin}
    \begin{aligned}
        \bfu_k & := \underset{s\in\{\pm 1\}}{\arg\min}\norm{s\cdot \texttt{nev}_1\paren{\bm{\Sigma}_k} - \bfu_k^*}_2\cdot \texttt{nev}_1\paren{\bm{\Sigma}_k}.
    \end{aligned}
\end{equation}
The difference between $\bfv_k$, the $k$-th approximate eigenvector returned by Algorithm~\ref{alg:main-alg}, and $\bfu_k^*$, the $k$-th ground-truth eigenvector, consists of two components: $i)$ $\norm{\bfu_k - \bfu_k^*}_2$, the difference between the top eigenvectors of the empirical deflated matrix $\bm{\Sigma}_k$ and the ``ground-truth'' deflated matrix $\bm{\Sigma}_k^*$; and $ii)$ $\norm{\bfv_k - \bfu_k}_2$, the numerical error from not solving the sub-routine $\texttt{PCA}(\cdot)$ exactly. 
We can decompose the difference $\norm{\bfv_k - \bfu_k^*}_2$ in the following sense:
\begin{equation}
    \label{eq:eig_err_decomp}
    \norm{\bfv_k - \bfu_k^*}_2 \leq \underbrace{\norm{\bfv_k - \bfu_k}_2}_{:=\norm{\bm{\delta}_k}_2} + \norm{\bfu_k - \bfu_k^*}_2.
\end{equation}
where the last equality follows from the definition in (\ref{eq:def-delta-v}). The norm of $\norm{\bm{\delta}_k}_2$ depends largely on the sub-routine $\texttt{PCA}$ and can be controlled as long as $t$, the number of sub-routine iterations, is large enough. Therefore, our analysis will mainly focus on upper-bounding $\norm{\bfu_k - \bfu_k^*}_2$. 
Intuitively, $\norm{\bfu_k - \bfu_k^*}_2$ depends on the difference between $\bm{\Sigma}_k$ and $\bm{\Sigma}_k^*$. As a building block of our proof, we will first provide a characterization of $\norm{\bm{\Sigma}_k - \bm{\Sigma}_k^*}_F$.

\begin{lemma}
    \label{lem:matrix_diff_propagate}
    Suppose that $\norm{\bm{\delta}_k}_2\leq \frac{1}{6}$ for all $k\in[K-1]$, then we have that for all $k\in[K-1]$,
    \begin{equation}
        \label{eq:matrix_diff_propagate}
        \begin{aligned}
        \norm{\bm{\Sigma}_{k+1} - \bm{\Sigma}_{k+1}^*}_F & \leq 3\norm{\bm{\Sigma}_k - \bm{\Sigma}_k^*}_F + 5\lambda_k^*\norm{\bm{\delta}_k}_2\\
        & \quad\quad+ 2\lambda_k^*\norm{\bfu_k - \bfu_k^*}_2.
        \end{aligned}
    \end{equation}
\end{lemma}

\noindent Lemma~\ref{lem:matrix_diff_propagate} upper-bounds $\norm{\bm{\Sigma}_{k+1} - \bm{\Sigma}_{k+1}^*}_F$ in terms of $\norm{\bm{\Sigma}_k - \bm{\Sigma}_k^*}_F, \norm{\bfu_k - \bfu_k^*}_2$, and $\norm{\bm{\delta}_k}_2$. To establish an recursive characterization of $\norm{\bm{\Sigma}_k -\bm{\Sigma}_k^*}_F$, we need to obtain an upper bound for $\norm{\bfu_k - \bfu_k^*}_2$. Indeed, this can be easily obtained by applying the Davis-Kahan $\sin\Theta$ theorem \cite{davis1970rotation, eldridge2017unperturbed}.
\begin{lemma}
    \label{lem:top_eig_diff}
    Let $\mathcal{T}_{k} := \min_{j\neq k}\left|\lambda_k^* - \lambda_j^*\right|$. If $\norm{\bm{\Sigma}_k - \bm{\Sigma}_k^*}_F\leq \frac{1}{4}\mathcal{T}_k$, then we shall have that
    \begin{equation}
        \label{eq:top_eig_diff}
        \norm{\bfu_k - \bfu_k^*}_2 \leq \frac{2}{\mathcal{T}_k}\norm{\bm{\Sigma}_k - \bm{\Sigma}_k^*}_F.
    \end{equation}
\end{lemma}
Plugging (\ref{eq:top_eig_diff}) into (\ref{eq:matrix_diff_propagate}) gives a recurrence that depends purely on $\norm{\bm{\delta}_k}_2$ and the spectrum of $\bm{\Sigma}$:
\begin{align*}
    \norm{\bm{\Sigma}_{k+1} - \bm{\Sigma}_{k+1}^*}_F & \leq \paren{3 + \frac{2}{\mathcal{T}_k}}\norm{\bm{\Sigma}_k - \bm{\Sigma}_k^*}_F\\
    & \quad\quad + 5\lambda_k^*\norm{\bm{\delta}_k}_2.
\end{align*}
Unrolling this recurrence gives a closed-form upper bound for $\norm{\bm{\Sigma}_k - \bm{\Sigma}_k^*}_F$. Combining this upper bound with (\ref{eq:top_eig_diff}) and plugging the result into (\ref{eq:eig_err_decomp}) gives us an upper bound for $\norm{\bfv_k - \bfu_k^*}_2$.

\vspace{-0.2cm}
\section{Error Propagation When Using Power Iteration}
\label{sec:main_result_pi}
In this section, we turn our focus to a specific sub-routine algorithm ($\texttt{PCA}$ in line 5 of Algorithm~\ref{alg:main-alg}) for finding the top eigenvectors: the power iteration method \citep{muntz1913solution}.
Given a symmetric matrix $\bfM\in\R^{d\times d}$ with rank $r$ and an initialization vector $\bfx_0\in\R^d$, the power iteration executes:
\begin{equation}
    \label{eq:power_iteration}
    \vspace{-0.3cm}
    \hat{\bfx}_{t+1} = \bfM\bfx_t;\quad \bfx_{t+1} = \frac{\hat{\bfx}_{t+1}}{\norm{\hat{\bfx}_{t+1}}_2}.
\end{equation}
Assuming that $\texttt{PCA}$ is the power iteration not only gives us the convergence rate of $\norm{\bm{\delta}_k}_2$ in terms of the number of steps in each power iteration, but also provides the directional information $\bm{\delta}_k$. In particular, one could see that $\left|\bm{\delta}_k^\top\bfu_{k+j-1}\right| \leq c_0\paren{\tfrac{\lambda_j\paren{\Sigma}_k}{\lambda_k^*}}^t$ for some constant $c_0$\footnote{For the detail of proving this statement please see Lemma~\ref{lem:pw_err_eig_align}}. Namely, the component of $\bm{\delta}_k$ along $\bfu_j$ has an exponentially smaller magnitude if $j$ is larger, depending on the eigen spectrum of $\bm{\Sigma}^*$. By leveraging this additional information, our theorem below shows that we can improve the bound on $\norm{\bfv_k - \bfu_k^*}_2$, when $\texttt{PCA}$ is specified to the power iteration, compared to the sub-routine agnostic result in Theorem~\ref{thm:main_theorem_1}.
\begin{theorem} 
    \label{thm:power-iteration}
    Consider the problem of looking for the top-$K$ eigenvectors of $\bm{\Sigma} \in\R^{d\times d}$ under Assumption~\ref{assu:main-assumption}. Let $\mathcal{T}_i = \lambda_i^* - \lambda_{i+1}^*$ for $i\in[d-1]$ and $\mathcal{T}_d = \lambda_d^*$, and let $\mathcal{T}_{k,\min} = \min_{i\in[k]}\mathcal{T}_i$. Assume that there exists a constant $c_0  \in (0,\infty)$ such that the initialized vector $\bfx_{0,k}$ in the $k$-th power iteration procedure satisfies $\big|\bfx_{0,k}^\top\bfu_{k}\big| > c_0^{-1}$ for all $k\in[K]$.\footnote{This is a standard assumption to guarantee the convergence of the power iteration, as in \citet{desa2017accelerated, GoluVanl96}} Let $\mathcal{G} =  \max_{k\in[K]}\paren{1 + \frac{c_0\lambda_k^*\lambda_{k+1}^*}{\lambda_k^*-\lambda_{k+1}^*}}$. If we use $t$ steps of power iteration to solve for the top eigenvectors, where $t$ satisfies:
    \begin{equation}
        \label{eq:thm_pw_t_req_0}
        t \geq \begin{cases}
            \log2\mathcal{G}_k\paren{\log \frac{\lambda_{k'+1}^* + 7\lambda_{k'}^*}{\lambda_{k'}^* + 7\lambda_{k'+1}^*}}^{-1}, & \forall k\leq K\\
            \frac{1}{\log \lambda_{k}^* - \log \lambda_{k+1}^*}, &  \forall k\leq d
        \end{cases}
    \end{equation}
    Moreover, if $t$ is also large enough to guarantee that
    \begin{equation}
        \label{eq:thm_pw_t_req}
        \sum_{k'=1}^{K-1}8^{K-k'}\frac{\lambda_{k'}^*}{\mathcal{T}_{k'}}\paren{\frac{7\lambda_{k+1}^* + \lambda_k^*}{7\lambda_k^* + \lambda_{k+1}^*}}^t \leq \frac{\mathcal{T}_{K,\min}}{140c_0},
    \end{equation}
    Then we can guarantee that for all $k\in[K]$
    \begin{equation}
        \label{eq:thm_pw}
        \begin{aligned}
        & \norm{\bfv_k - \bfu_k^*}_2 \\
        &\quad\quad\leq3\sum_{k'=1}^{k}8^{k-k'}\frac{\lambda_{k'}^*}{\lambda_k^*}\paren{5\norm{\bm{\delta}_{k'}}_2 + \frac{7c_0}{\mathcal{T}_{k'}}\paren{\frac{\lambda_{k'+1}^*}{\lambda_{k'}^*}}^t}.
        \end{aligned}
    \end{equation}
\end{theorem}
\begin{figure}
    \centering
  \includegraphics[width=\linewidth]{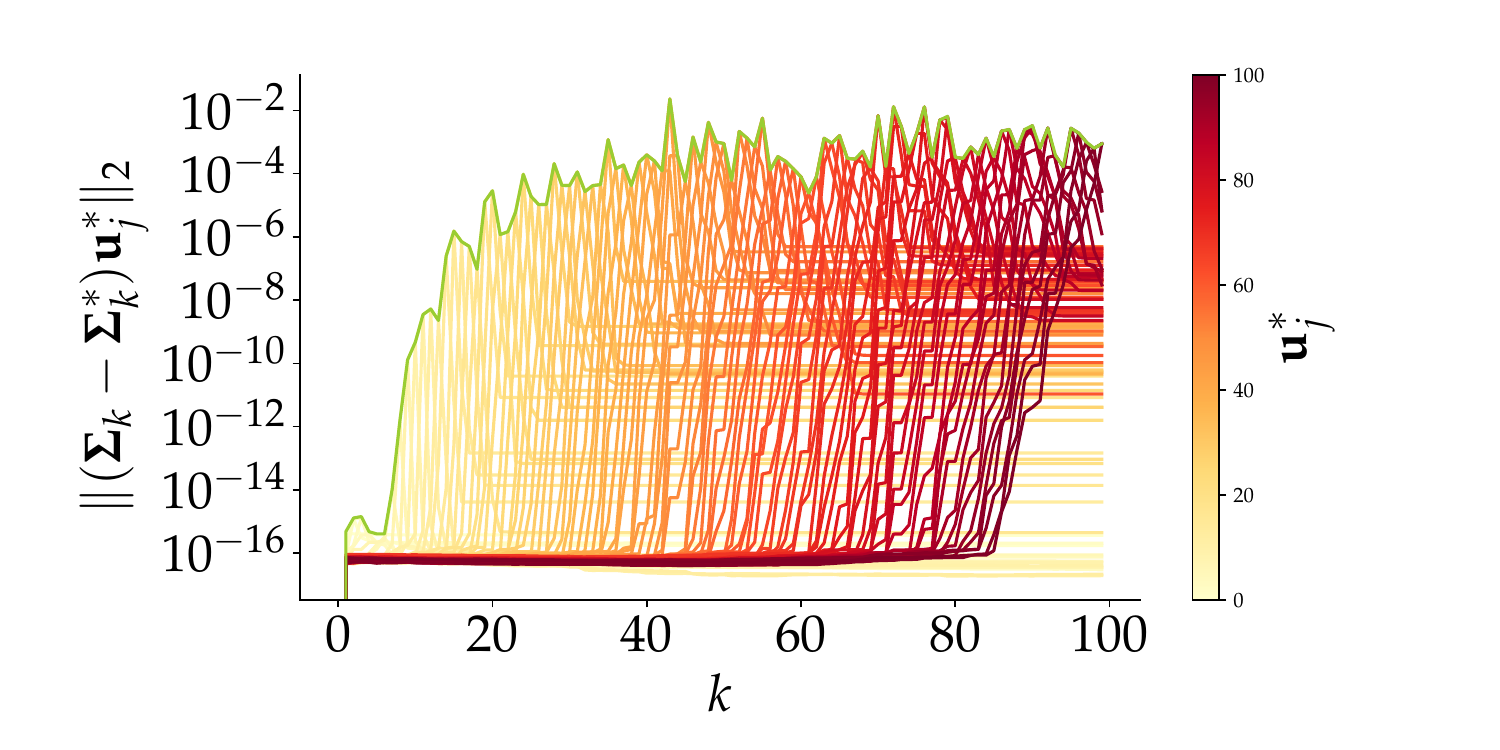}
  \vspace{-0.5cm}
  \caption{\textit{Dynamics of $\norm{\paren{\bm{\Sigma}_k - \bm{\Sigma}_k^*}\bfu_j^*}_2$ with respect to the change of $k$}. Each $\bfu_j^*$ is represented by a different color, with light color for small $j$ and dark color for large $j$. Experiments done for $\bm{\Sigma}\in\R^{100\times 100},\lambda_k^* = \frac{1}{k},\left\{\bfu_k^*\right\}_{k=1}^d$ being randomly generated orthogonal basis, and $t = 200$. The orthogonal basis $\left\{\bfu_k^*\right\}_{k=1}^d$ is generated by randomly sampling a matrix with I.I.D. Gaussian entries, and computing its left singular vectors.}
  \label{fig:mat_diff_align}
\end{figure}
\textbf{Remark 5.}
Similar to Theorem~\ref{thm:main_theorem_1}, the requirement of $t$ in (\ref{eq:thm_pw_t_req}) is to guarantee that $\norm{\bm{\Sigma}_k- \bm{\Sigma}_k^*}_F \leq \frac{1}{8}\mathcal{T}_{K,\min}$. In particular, (\ref{eq:thm_pw_t_req}) combined with (\ref{eq:thm_pw_t_req_0}) can be satisfied by the following simplified condition on $t$
\begin{equation}
    t \geq \begin{cases}
        \paren{\frac{\max\left\{\log 2\mathcal{G},3(K-k) + \log \tfrac{140c_0K}{\mathcal{T}_{K,\min}^2}\right\}}{\log\paren{7\lambda_{k}^* + \lambda_{k+1}^*}-\log\paren{7\lambda_{k+1}^* + \lambda_{k}^*}}}, & \forall k\leq K\\
        \frac{1}{\log \lambda_{k}^* - \log \lambda_{k+1}^*}, &  \forall k\leq d
    \end{cases}.
\end{equation}
The upper bound on $\norm{\bfv_k - \bfu_k^*}_2$ in (\ref{eq:thm_pw}) also takes a similar form to (\ref{eq:thm1_conclusion}). Recall that when $\texttt{PCA}$ is chosen to be the power iteration, we have $\norm{\bm{\delta}_k}_2 \leq c_0\paren{\tfrac{\lambda_{2}\paren{\bm{\Sigma}_k}}{\lambda_k}}^t\approx c_0\paren{\frac{\lambda_{k+1}^*}{\lambda_k^*}}^t$. Therefore, one can treat $\paren{5\norm{\bm{\delta}_{k'}}_2 + \tfrac{7c_0}{\mathcal{T}_{k'}}\paren{\tfrac{\lambda_{k'+1}^*}{\lambda_{k'}^*}}^t}$ in (\ref{eq:thm_pw}) as roughly $\tfrac{C}{\mathcal{T}_{k'}}\norm{\bm{\delta}_{k'}}_2$ for some constant $C$. Thus, the $\prod_{j=k+1}^K\paren{3 + \tfrac{2\lambda_j^*}{\mathcal{T}_j}}$ scaling of each summand in (\ref{eq:thm1_conclusion}) that depends on $\mathcal{T}_j$ is improved to the exponential scaling of $8^{K-k}$ in (\ref{eq:thm_pw}). This improvement is further made clear in the corollary below.
\begin{corollary} \label{cor:power-iteration2}
    Under the same assumptions as in Theorem~\ref{thm:power-iteration}, if we use $t$ steps of power iteration to solve for the top eigenvectors such that:
    \begin{equation}
        \label{eq:pw_t_req}
        t \geq \begin{cases}
            \Omega\paren{\frac{\max\left\{\log\mathcal{G},K-k + \log \frac{c_0K}{\epsilon\mathcal{T}_{K,\min}}\right\}}{\log\paren{7\lambda_{k}^* + \lambda_{k+1}^*} - \log\paren{7\lambda_{k+1}^* + \lambda_{k}^*}}}, &  \forall k\leq K, \text{ and}\\
            \Omega\paren{\frac{1}{\log \lambda_{k}^* - \log \lambda_{k+1}^*}}, &  \forall k\leq d,
        \end{cases}
    \end{equation}
    then we can guarantee $\norm{\bfv_k - \bfu_k}_2\leq \varepsilon$ for all $k\in[K]$.
\end{corollary}
\textbf{Remark 6.}
Corollary~\ref{cor:power-iteration2} is obtained by setting $\norm{\bm{\delta}_k}_2 \leq c_0\paren{\tfrac{\lambda_{2}\paren{\bm{\Sigma}_k}}{\lambda_k}}^t$ in power iteration and notice that $\tfrac{\lambda_{2}\paren{\bm{\Sigma}_k}}{\lambda_k} \leq \tfrac{7\lambda_{k+1}^* + \lambda_k^*}{7\lambda_k^* + \lambda_{k+1}^*}$ when $\norm{\bm{\Sigma}_k - \bm{\Sigma}_k^*}_2\leq \tfrac{1}{8}\mathcal{T}_k$. When we treat $c_0$ and $\left\{\lambda_k^*\right\}_{k=1}^d$ as constants, $\mathcal{G}$ also becomes a constant. In this case, the lower bound on $t$ in (\ref{eq:pw_t_req}) becomes $t \geq \Omega\paren{\log\frac{K}{\epsilon} + K}$. Compared with the $t\geq \Omega\paren{\log\frac{K}{\epsilon} + K\log K}$ requirement we derived in the discussion of Corollary~\ref{cor:main_cor_t}.
Thus, our lower bound in Corollary~\ref{cor:power-iteration2} saves a factor of $\log K$.

\medskip
\begin{figure*}
    \centering
    \includegraphics[width=\linewidth]{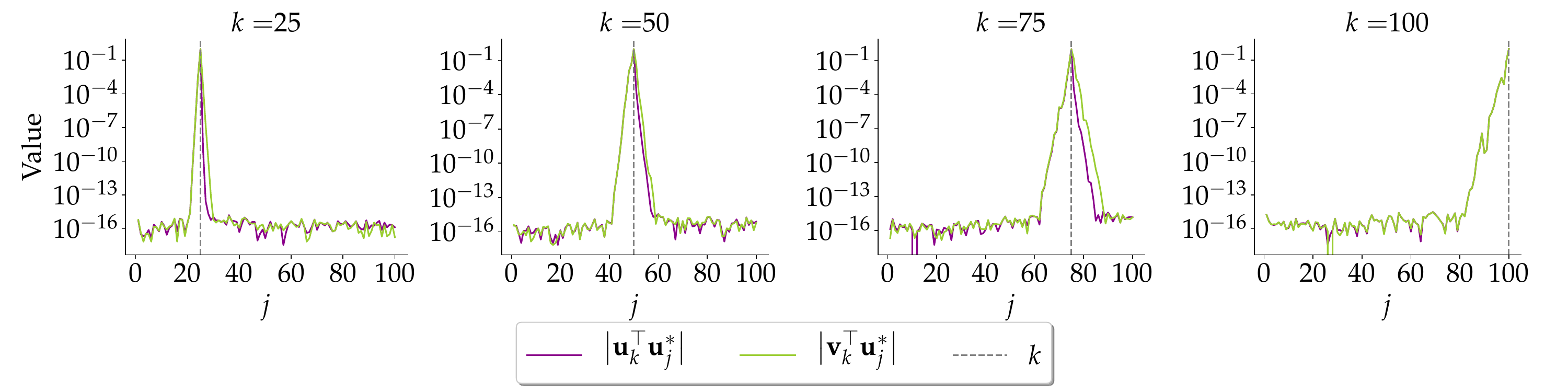}
    \caption{The comparison among the dynamics of $\bfu_k^\top\bfu_j^*$ and $\bfv_k^\top\bfu_j^*$ with respect to the change of $k$ for $j\in\{25, 50, 75, 100\}$. Experiments are performed for $\bm{\Sigma}\in\R^{100\times 100},\lambda_k^* = \frac{1}{k},\left\{\bfu_k^*\right\}_{k=1}^d$ being randomly generated orthogonal basis, with $t = 200$. The results show that both $\left|\bfu_k^\top\bfu_j^*\right|$ and $\left|\bfv_k^\top\bfu_j^*\right|$ are small only when $k$ is near $j$. The orthogonal basis $\left\{\bfu_k^*\right\}_{k=1}^d$ is generated by randomly sampling a matrix with I.I.D. Gaussian entries, and computing its left singular vectors.}
    \label{fig:u_k_start_v_k_star}
    \vspace{-0.3cm}
\end{figure*}
\noindent\textbf{Remark 7.} In the discussion about Theorem~\ref{thm:power-iteration} and Corollary~\ref{cor:power-iteration2} above, we deliberately ignored the effect of $c_0$ and $\mathcal{G}$. For $\mathcal{G}$, we observe that $\mathcal{G}$ can also be written as $\mathcal{G} = \max_{k\in[K]}\paren{1 + c_0\paren{\frac{1}{\lambda_{k+1}^*} - \frac{1}{\lambda_k^*}}^{-1}}$. If the eigenvalues of $\bm{\Sigma}$ follow a power law decay as $\lambda_k^* = \frac{1}{k^\gamma}$ for some $\gamma \geq 1$, then we have $\frac{1}{\lambda_{k+1}^*} - \frac{1}{\lambda_k^*}\geq 1$ and thus $\mathcal{G} \leq c_0 + 1$. Next, we focus on $c_0$. If in each power iteration sub-routine, the initialized vector $\bfx_{0,k}$ is generated randomly, then $\big|\bfx_{0,k}^\top\bfu_k\big|$ is independent of $\bfu_k$, and $\mathbb{P}\paren{\big|\bfx_{0,k}^\top\bfu_k\big| > 0} = 1$. This idea can be made concrete in the following lemma
\begin{lemma}
    \label{lem:init_property}
    Let $\bfx_{0,k}$ be sampled uniformly from a unit sphere in $\R^d$. Then with probability at least $1 - \frac{2K}{3d}$ we have that $\big|\bfx_{0,k}^\top\bfu_k\big|\geq \frac{1}{\sqrt{1 + 2d^3}}$ for all $k\in[K]$.
\end{lemma}
The proof of Lemma~\ref{lem:init_property} follows the idea of \citep{wang2020nearly}, and is deferred to Appendix~\ref{sec:proof_init_property}. Next, we will first explore the property of power iteration crucial to our proof of Theorem~\ref{thm:power-iteration} in Section~\ref{sec:power_iteration}, and then provide an overview of the proof in Section~\ref{sec:proof_overview_pw}.

\subsection{What Does Power Iteration Tell Us?}
\label{sec:power_iteration}
Recall our previous observation that the error vector returned by power iteration has different magnitudes along different eigenvector directions of the matrix of interest.
In this section, we perform the same analysis, assuming that the matrix passed into the power iteration sub-routine is a perturbation of some ``ground-truth'' matrix. Let $\bfa_j^*$ be the $j$-th eigenvector of $\bfM^*$ and $\bfM = \bfM^* + \bfH$. Given the definition of $\bfx_t$ as the result of a $t$-step power iteration on $\bfM$, defined in (\ref{eq:power_iteration}), we will show that $\bfx_t^\top\bfa_j^*$'s take different forms depending on the index of the eigenvector $j$.
\begin{lemma}
    \label{lem:pw_mat_perturb}
    Let $\bfx_t$ be the result of running power iteration starting from $\bfx_0$ for $t$ iterations, as defined in (\ref{eq:power_iteration}). Let $\bfM = \bfM^* + \bfH$. Moreover, let $\sigma_{j}^*, \bfa_j^*$ be the $j$-th eigenvalue and eigenvector of $\bfM^*$, and let $\sigma_j$ be the $j$-th eigenvalue of $\bfM$. 
    Assume that there exists $c_0 \in (0,\infty)$ such that $\left|\bfx_0^\top\bfa_1\right|\geq c_0^{-1}$. Then, for all $j=2,\dots,r$, we have:
    \begin{equation}
        \label{eq:pw_mat_perturb_c2}
        \left|\bfx_t^\top\bfa_j^*\right| \leq c_0\paren{\paren{\frac{\sigma_j^{*}}{\sigma_1}}^t + \frac{\sigma_j^*}{\sigma_1 - \sigma_j^*}\norm{\bfH\bfa_j ^*}_2}.
    \end{equation}
\end{lemma}
Lemma~\ref{lem:pw_mat_perturb} upper bounds the alignment of the vector in the $t$th step of the power iteration with any incorrect eigenvector $\bfa_j^*$. The magnitude of $\bfx_t^\top\bfa_j^*$ depends both on $\paren{\frac{\sigma_j^{*}}{\sigma_1}}^t$ and how the perturbation matrix aligns with the $j$-th eigenvector. When $\bfM^*$ has a decaying spectrum, the first term in the upper bound of $\left|\bfx_t^\top\bfa_j^*\right|$ in (\ref{eq:pw_mat_perturb_c2}) decreases as $j$ becomes larger. Overloading Lemma~\ref{lem:pw_mat_perturb} in the notations of our deflation procedure, we have
\[
    \left|\bfv_k^\top\bfu_j^*\right|\leq c_0\paren{\paren{\frac{\lambda_j^*}{\lambda_1}}^t + \frac{\lambda_j^*}{\lambda_1 - \lambda_j^*}\norm{\paren{\bm{\Sigma}_k - \bm{\Sigma}_k^*}\bfu_j^*}_2}
\]
In short, the component of $\bfv_k$ along an incorrect director, $\left|\bfv_k^\top\bfu_j^*\right|$, is determined by both the convergence along the $j$th eigenvector, and by the difference between the two deflated matrix along the direction of $\bfu_j^*$. This calls for analysis on $\norm{\paren{\bm{\Sigma}_k - \bm{\Sigma}_k^*}\bfu_j^*}_2$. Indeed, since the matrix residue $\bm{\Sigma}_k - \bm{\Sigma}_k^*$ is resulted from $\bfv_k - \bfu_k$'s, it is natural to believe that $\norm{\paren{\bm{\Sigma}_k - \bm{\Sigma}_k^*}\bfu_j^*}_2$ can also be upper-bounded by quantities depending on $j$. Our next lemma characterizes this behavior
\begin{lemma}
    \label{lem:matrix_diff_align}
    Let $\bfv_k$ be the output of the $k$-th power iteration. For all $k\in[d]$ and $j\geq k$, we have:
    \begin{equation}
        \label{eq:mat_diff_align_bound}
        \norm{\paren{\bm{\Sigma}^*_k-\bm{\Sigma}_k}\bfu_j^*}_2 \leq \sum_{k'=1}^{k-1}\lambda_{k'}\left|\bfv_{k'}^\top\bfu_j^*\right|.
    \end{equation}
\end{lemma}
\begin{figure*}
    \centering
    \includegraphics[width=\linewidth]{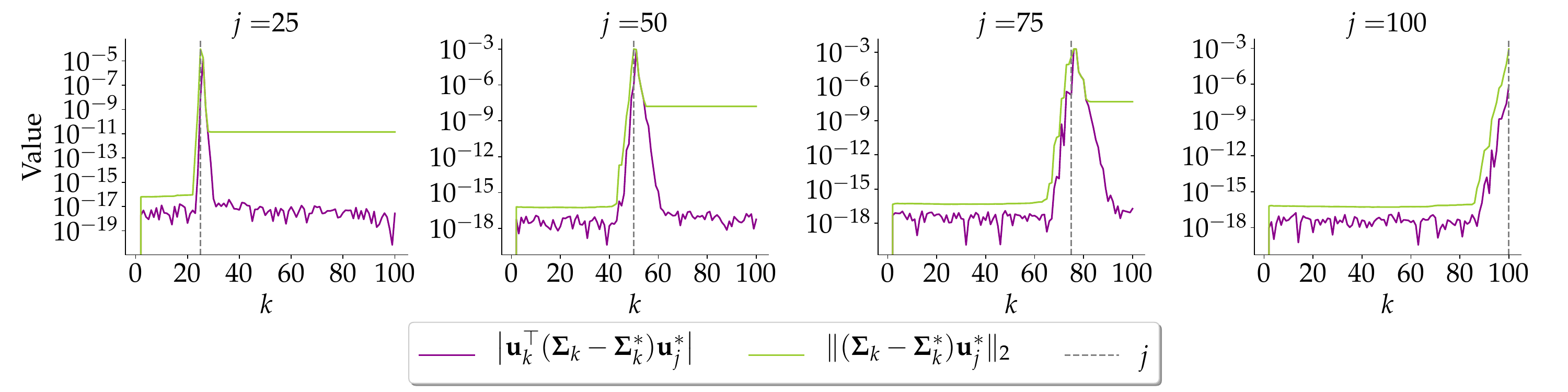}
    \caption{The comparison of between the dynamics of $\left|\bfu_k^\top\paren{\bm{\Sigma}_k - \bm{\Sigma}_k^*}\bfu_k^*\right|$ and $\norm{\paren{\bm{\Sigma}_k - \bm{\Sigma}_k^*}\bfu_j^*}_2$ with respect to the change of $k$ for $j\in\{25, 50, 75, 100\}$. Experiments are performed for $\bm{\Sigma}\in\R^{100\times 100},\lambda_k^* = \frac{1}{k},\left\{\bfu_k^*\right\}_{k=1}^d$ as randomly generated orthogonal basis, with $t = 200$. The results show that $\norm{\paren{\bm{\Sigma}_k - \bm{\Sigma}_k^*}\bfu_j^*}_2$ is a good approximation of $\left|\bfu_k^\top\paren{\bm{\Sigma}_k - \bm{\Sigma}_k^*}\bfu_k^*\right|$ when the latter is large. The orthogonal basis $\left\{\bfu_k^*\right\}_{k=1}^d$ is generated by randomly sampling a matrix with I.I.D. Gaussian entries, and computing its left singular vectors.}
    \label{fig:inner_norm_compare}
    \vspace{-0.5cm}
\end{figure*}
The proof of Lemma~\ref{lem:matrix_diff_align} is deferred to Appendix~\ref{sec:proof_matrix_diff_align}. In particular, (\ref{eq:mat_diff_align_bound}) in Lemma~\ref{lem:matrix_diff_align} upper-bounds $\big\|\paren{\bm{\Sigma}^*_k-\bm{\Sigma}_k}\bfu_j^*\big\|_2$ using a linear combination of $\big\{\big|\bfv_{k'}^\top\bfu_j^*\big|\big\}_{k'=1}^{k-1}$. Indeed, recall that $\bfv_k$ is the output of the $k$-th power iteration procedure. This means that we can invoke Lemma~\ref{lem:pw_mat_perturb} with $\bfM = \bm{\Sigma}_k, \bfM^* = \bm{\Sigma}_k^*$, and correspondingly $\bfH = \bm{\Sigma}^*_k-\bm{\Sigma}_k$. 
Plugging the bound of $\big\|\paren{\bm{\Sigma}^*_k-\bm{\Sigma}_k}\bfu_j^*\big\|_2$ from (\ref{eq:mat_diff_align_bound}) into then gives the recursively defined $\big|\bfv_k^\top\bfu_j^*\big|$:
\begin{equation}
    \label{eq:vk_uj_growth2}
    \left|\bfv_k^\top\bfu_j^*\right| \leq c_0\paren{\paren{\frac{\lambda_j^*}{\lambda_k}}^t + \frac{\lambda_j^*}{\lambda_k - \lambda_j^*}\sum_{k'=1}^{k-1}\lambda_{k'}\left|\bfv_{k'}^\top\bfu_j^*\right|}.
\end{equation}
(\ref{eq:vk_uj_growth2}) shows the independence of each $\big\{\big|\bfv_k^\top\bfu_j^*\big|\big\}_{k=1}^{d}$ for different $j$, which allows each $\big\{\big|\bfv_k^\top\bfu_j^*\big|\big\}_{k=1}^{d}$ to grow at different speeds. To see whether this characterization of independent growth is crucial, observe that in Figure~\ref{fig:u_k_start_v_k_star}, each $\left|\bfv_k^\top\bfu_j^*\right|$ only goes through a fast growth when $j$ is near $k$. This behavior results in the phenomena that at each step $k$, there are only a few $j$'s where $\big\|\paren{\bm{\Sigma}^*_k-\bm{\Sigma}_k}\bfu_j^*\big\|_2$ is large. We can see this as we unroll this recursive form and obtain
\begin{equation}
    \label{eq:mat_diff_align_closed_form}
    \begin{aligned}
    \norm{\paren{\bm{\Sigma}^*_k-\bm{\Sigma}_k}\bfu_j^*}_2 & \leq \sum_{k'=1}^{k-1}\lambda_{k'}\left|\bfv_{k'}^\top\bfu_j^*\right| \\
    & \leq c_0\sum_{k'=1}^{k-1}\mathcal{G}_k^{k-k'-1}\paren{\frac{\lambda_j^*}{\lambda_k}}^{t-1}
    \end{aligned}
\end{equation}
where $\mathcal{G}_k = 1 + \frac{c_0\lambda_k^*\lambda_{k+1}^*}{\lambda_k^*-\lambda_{k+1}^*}$. (\ref{eq:mat_diff_align_closed_form}) upper bounds $\norm{\paren{\bm{\Sigma}^*_k-\bm{\Sigma}_k}\bfu_j^*}_2$ with a summation of the linear decay terms resulted from the power iteration. For the detailed computation of unrolling the recursive definition, we refer the reader to Appendix~\ref{sec:proof_power_iteration} and Appendix~\ref{sec:auxiliary_lem} (see Lemma~\ref{lem:lem:sum_acc_seq_growth_exact}). Noticeably, (\ref{eq:mat_diff_align_closed_form}) provides a closed-form upper bound of $\big\|\paren{\bm{\Sigma}^*_k-\bm{\Sigma}_k}\bfu_j^*\big\|_2$, which further demonstrates the independent growth of $\big\|\paren{\bm{\Sigma}^*_k-\bm{\Sigma}_k}\bfu_j^*\big\|_2$ as $k$ increases for different $j$. This behavior corresponds to what we observed in experiments in Figure~\ref{fig:mat_diff_align}, where each $\big\|\paren{\bm{\Sigma}^*_k-\bm{\Sigma}_k}\bfu_j^*\big\|_2$ grows sequentially.
\subsection{Proof Overview of Theorem~\ref{thm:power-iteration}}
\label{sec:proof_overview_pw}
In this section, we will demonstrate how the properties of power iteration, in particular (\ref{eq:mat_diff_align_closed_form}) is used to construct a proof for Theorem~\ref{thm:power-iteration}. We start with the following lemma.
\begin{lemma}
    \label{lem:eigen_inner_equal}
    Let $\bfM^*\in\R^{d\times d}$ be a rank-$r$ matrix and let $\bfM = \bfM^* + \bfH$ for some $\bfH\in\R^{d\times d}$. Let $\sigma_i^*, \bfa_i^*$ be the $i$-th eigenvalue and eigenvector of $\bfM^*$, and $\sigma_i, \bfa_i$ be the $i$-th eigenvalue and eigenvector of $\bfM$. Then, for all $i,j\in[r]$ with $\sigma_i\neq \sigma_j^*$, we have:
    \vspace{-0.3cm}
    \[
        \bfa_i^\top\bfa_j^* = \frac{\bfa_i^\top\bfH\bfa_j^*}{\sigma_i - \sigma_j^*}.
    \]
    \vspace{-0.7cm}
\end{lemma}
Unfortunately, we cannot directly characterize $\bfu_k^\top\bfu_k^*$ using Lemma~\ref{lem:eigen_inner_equal} by plugging in $i = j = 1$ and $\bfM = \bm{\Sigma}_k,\bfM^* = \bm{\Sigma}_k^*$, since, as far as we know, it is impossible to show that $\sigma_i = \lambda_k \neq \lambda_k^* = \sigma_j^*$. However, Lemma~\ref{lem:eigen_inner_equal} will be useful to upper-bound $\paren{\bfu_k^\top\bfu_j^*}^2$ when $j > k$, since in this case, we can guarantee that $\lambda_k \geq \lambda_k^* - \norm{\bm{\Sigma}_k - \bm{\Sigma}_k^*}_2 > \lambda_j^*$ as long as $\norm{\bm{\Sigma}_k - \bm{\Sigma}_k^*}_2 < \mathcal{T}_{k}$, where $\mathcal{T}_k$ is the $k$-th eigengap of $\bm{\Sigma}$. Under such conditions, we have:
\vspace{-0.1cm}
\begin{equation}
    \label{eq:uk_uj_inner_abs}
    \left|\bfu_k^\top\bfu_j^*\right| \leq \frac{\norm{\paren{\bm{\Sigma}_k - \bm{\Sigma}_k^*}\bfu_j^*}_2}{\lambda_k - \lambda_j^*}.
    \vspace{-0.1cm}
\end{equation}
Indeed, going from Lemma~\ref{lem:eigen_inner_equal} to (\ref{eq:uk_uj_inner_abs}) loosens the bound by ignoring the directional information of $\bfu_k$. However, through experiments, we observe that (\ref{eq:uk_uj_inner_abs}) is a good enough upper bound since $\norm{\paren{\bm{\Sigma}_k - \bm{\Sigma}_k^*}\bfu_j^*}_2$ already stays at a much smaller magnitude than $\norm{\bm{\Sigma}_k - \bm{\Sigma}_k^*}_2$. This is reflected in Figure~\ref{fig:inner_norm_compare}, where we can see that the curve $\big\|\paren{\bm{\Sigma}_k - \bm{\Sigma}_k^*}\bfu_j^*\big\|_2$ enjoys the same decrease as $\big|\bfu_k^\top\paren{\bm{\Sigma}_k - \bm{\Sigma}_k^*}\bfu_j^*\big|$ when $j$ first surpasses $k$, and stops decreasing after reaching a small magnitude (lower than $10^{-5}$). (\ref{eq:uk_uj_inner_abs}) successfully applied the property of the power iteration in (\ref{eq:mat_diff_align_closed_form}) to bound each $\left|\bfu_k^\top\bfu_j^*\right|$. Our next lemma connects the upper bound of each $\left|\bfu_k^\top\bfu_j^*\right|$ to a lower bound of $\left|\bfu_k^\top\bfu_k^*\right|$.
\begin{lemma}
    \label{lem:inner_ub_to_inner_lb}
    Suppose that $\norm{\bm{\Sigma}_k - \bm{\Sigma}_k^*}_F\leq \frac{1}{8}\lambda_k^*$, then we have
    \vspace{-0.3cm}
    \begin{equation}
        \label{eq:inner_ub_to_inner_lb}
        \paren{\bfu_k^\top\bfu_k^*}^2 \geq 1 - \frac{2.4}{\lambda_k^{*2}}\norm{\bm{\Sigma}_k - \bm{\Sigma}_k^*}_F^2 - \sum_{j=k+1}^d\paren{\bfu_k^\top\bfu_j^*}^2
    \end{equation}
    \vspace{-0.7cm}
\end{lemma}
Our proof of Lemma~\ref{lem:inner_ub_to_inner_lb} is based on Neumann expansion \cite{eldridge2017unperturbed,chen2020asymmetry}, a powerful technique for analyses that involve the product between perturbation matrix and eigenvectors. Indeed, the summation of the last term of (\ref{eq:inner_ub_to_inner_lb}) may seem to introduce an additional factor of $d$. However, since the sequence $\left\{\left|\bfu_k^\top\bfu_j^*\right|\right\}_{j=k+1}^d$ decays fast, as observed in Figure~\ref{fig:u_k_start_v_k_star}, the summation will eventually be independent of $d$. Combining (\ref{eq:mat_diff_align_closed_form}), (\ref{eq:uk_uj_inner_abs}), and (\ref{eq:inner_ub_to_inner_lb}) gives the following lemma that upper bounds $\norm{\bfu_k - \bfu_k^*}_2$.
\begin{lemma}
    \label{lem:pw_top_eig_diff}
    Assume that there exists a constant $c_0  \in (0,\infty)$ such that the initialized vector $\bfx_{0,k'}$ in the $k'$-th power iteration procedure satisfies $\big|\bfx_{0,k'}^\top\bfu_{k'}\big| > c_0^{-1}$ for all $k'\in[k]$. Let $\mathcal{T}_k = \lambda_k^* - \lambda_{k+1}^*, \mathcal{T}_{k,\min} = \min_{k'\in[k]}\mathcal{T}_{k'}$ and $\mathcal{G}_k =  1 + \frac{c_0\lambda_k^*\lambda_{k+1}^*}{\lambda_k^*-\lambda_{k+1}^*}$. Suppose $\norm{\bm{\Sigma}_{k'} - \bm{\Sigma}_{k'}^*}_F\leq \frac{1}{8}\mathcal{T}_{k,\min}$, and $t \geq \log2\mathcal{G}_k\paren{\log \frac{\lambda_{k'+1}^* + 7\lambda_{k'}^*}{\lambda_{k'}^* + 7\lambda_{k'+1}^*}}^{-1}$
    for all $k'\in[k-1]$. Then, we have that:
    \vspace{-0.3cm}
    \[
        \norm{\bfu_k - \bfu_k^*}_2^2\leq \frac{4.8}{\lambda_k^*}\norm{\bm{\Sigma}_k-\bm{\Sigma}_k^*}_2^2 + \frac{11c_0^2}{2\mathcal{T}_k^2}\sum_{j=k+1}^d\paren{\frac{\lambda_j^*}{\lambda_{k}^*}}^{2t}.
    \]
    \vspace{-0.7cm}
\end{lemma}
Lemma~\ref{lem:pw_top_eig_diff} shows a similar result to Lemma~\ref{lem:top_eig_diff}. In particular, the first term in the upper bound of Lemma~\ref{lem:pw_top_eig_diff} also scales with $\norm{\bm{\Sigma}_k - \bm{\Sigma}_k^*}_2$, and the second term involves a summation over $\paren{\frac{\lambda_j^*}{\lambda_k^*}}^{2t}$ for $j > k$, which can be made small as long as $t$ is large enough. The upper bound in Lemma~\ref{lem:pw_top_eig_diff} shows a significant improvement in that the scaling factor in front of $\norm{\bm{\Sigma}_k - \bm{\Sigma}_k^*}_2$ depends only on $\lambda_k^*$ compared with the $O\paren{\frac{\sqrt{d-k}}{\mathcal{T}_k}}$ scaling in Lemma~\ref{lem:top_eig_diff}. This allows us to improve over the nearly factorial bound in Theorem~\ref{thm:main_theorem_1}.  After obtaining Lemma~\ref{lem:pw_top_eig_diff}, the rest of the proof for Theorem~\ref{thm:power-iteration} becomes similar to what we have discussed in Section~\ref{sec:proof_overview_thm1}. We can plug the upper bound in Lemma~\ref{lem:pw_top_eig_diff} into Lemma~\ref{lem:matrix_diff_propagate} and unroll the recurrence to obtain a closed-form upper bound for $\norm{\bm{\Sigma}_k - \bm{\Sigma}_k^*}_F$. After that, we can finish the proof by plugging the upper bound for $\norm{\bm{\Sigma}_k - \bm{\Sigma}_k^*}_F$ back into Lemma~\ref{lem:pw_top_eig_diff} and utilize (\ref{eq:eig_err_decomp}) to arrive at the desired upper bound for $\norm{\bfv_k - \bfu_k^*}_2$. The detailed proof of Theorem~\ref{thm:power-iteration} is deferred to Appendix~\ref{sec:proof_power_iteration}.

\section{Conclusion and Future Work}
Numerical algorithms for solving the top eigenvector of a given matrix often produce non-negligible errors, which will propagate and accumulate through multiple deflation steps. In this paper, we mathematically analyze this error propagation for a scenario agnostic to the sub-routine for solving the top eigenvectors and the case where this sub-routine is fixed to power iteration. In particular, for the sub-routine agnostic case, our analysis gives an exponential growth of the errors as one aims to solve more eigenvectors, where the base of the exponential growth depends on the eigengap. This upper bound on the errors is improved to exponential growth with a constant base by utilizing the directional information of the error vector produced by power iteration. Our result implies a lower bound on the number of power iteration steps required to solve for multiple leading eigenvectors using deflation up to a certain accuracy.

This paper focuses on the case where $\bm{\Sigma}$ is positive definite and has distinct eigenvalues. However, the analysis can also be extended to accommodate repeated and zero eigenvalues. When eigenvalues are repeated, the corresponding eigenvector and the method for measuring the error of the approximate eigenvector need to be redefined. Instead of considering a single eigenvector for each eigenvalue, we can generalize the approach by considering the subspace spanned by the eigenvectors corresponding to each repeated eigenvalue. The core ideas of our analysis can be applied to account for the error propagation between subspaces, but one must account for additional error propagation within subspaces. To extend the analysis to positive semi-definite matrices, $\bm{\Sigma}$ is of rank $r$ and we are interested in finding top-$K$ eigenvectors with $K \leq r$. In this case, our Theorem~\ref{thm:main_theorem_1} still holds, as it does not depend on eigenvalues beyond the $K$th. Theorem~\ref{thm:power-iteration} may require slight modification, but we hypothesize that the case where $\bm{\Sigma}$ is low rank will relax the condition on $t$ in (\ref{eq:thm_pw_t_req}) to $k\leq r$ instead of $k\leq d$.

While this paper focuses on the sub-routine agnostic scenario and the power iteration, future work can consider extending the analysis to the case where other algorithms, such as Lanczos method \citep{GoluVanl96} and power iteration with momentum \citep{desa2017accelerated}, are used to approximate the top eigenvectors in the deflation process.
\section*{Acknowledgements}
This work is supported by NSF CMMI no. 2037545, NSF CAREER award no. 2145629, Welch Foundation Grant \#A22-0307, a AWS Research Grant, a Microsoft Research Award, and a Rice Ken Kennedy Institute Seed Funding Award.

\section*{Impact Statement}
This paper aims to deepen the theoretical understanding of principal component analysis, a long-standing method in machine learning. The theoretical results presented herein may change the computation power used for principal component analysis. While the direct impact of this work may be limited to computational efficiency gains, the broader implications extend beyond this, potentially catalyzing advancements in various domains where principal component analysis finds application. However, no specific societal consequences can be directly attributed to the work itself.

\bibliography{example_paper}
\bibliographystyle{icml2024}

\newpage
\appendix
\onecolumn
\section{Missing Proofs from Section~\ref{sec:main_result_general}}
\label{sec:proof_main_result_general}

\subsection{Proof of Lemma~\ref{lem:matrix_diff_propagate}}
\label{sec:proof_matrix_diff_propagate}
Recall that $\bm{\Sigma}_{k+1} = \bm{\Sigma}_k - \bfv_k\bfv_k^\top\bm{\Sigma}_k\bfv_k\bfv_k^\top$. Since $\bfv_k = \bfu_k + \bm{\delta}_k$, we can write $\bm{\Sigma}_{k+1}$ as
\begin{align*}
    \bm{\Sigma}_{k+1} & = \bm{\Sigma}_k - \paren{\bfu_{k} + \bm{\delta}_k}\paren{\bfu_{k} + \bm{\delta}_k}^\top\bm{\Sigma}_k\paren{\bfu_{k} + \bm{\delta}_k}\paren{\bfu_{k} + \bm{\delta}_k}^\top\\
    & = \bm{\Sigma}_k - \paren{\bfu_{k} + \bm{\delta}_k}\paren{\bfu_{k}^\top\bm{\Sigma}_k\bfu_{k} + \bm{\delta}_k^\top\bm{\Sigma}_k\bfu_{k} + \bfu_{k}^\top\bm{\Sigma}_k\bm{\delta}_k + \bm{\delta}_k^\top\bm{\Sigma}_k\bm{\delta}_k}\paren{\bfu_{k} + \bm{\delta}_k}^\top\\
    & = \bm{\Sigma}_k - \paren{\lambda_k + 2\lambda_k\bfu_{k}^\top\bm{\delta}_k + \bm{\delta}_k^\top\bm{\Sigma}_k\bm{\delta}_k}\paren{\bfu_{k} + \bm{\delta}_k}\paren{\bfu_{k} + \bm{\delta}_k}^\top\\
    & = \underbrace{\bm{\Sigma}_k - \lambda_k\bfu_{k}\bfu_{k}^\top}_{\tilde{\bm{\Sigma}}_{k+1}} - \underbrace{\lambda_k\paren{\bm{\delta}_k\bfu_{k}^\top + \bfu_{k}\bm{\delta}_k^\top + \bm{\delta}_k\bm{\delta}_k^\top}}_{\bm{\mathcal{E}}_{1,k}}\\
    &\quad\quad\quad - \underbrace{\paren{2\lambda_k\bfu_{k}^\top\bm{\delta}_k + \bm{\delta}_k^\top\bm{\Sigma}_k\bm{\delta}_k}\paren{\bfu_{k} + \bm{\delta}_k}\paren{\bfu_{k} + \bm{\delta}_k}^\top}_{\bm{\mathcal{E}}_{2,k}}
\end{align*}
For the convenience of the analysis, we let $\Delta_k = \norm{\bm{\Sigma}_k - \bm{\Sigma}_k^*}_F$. In this way, we can write
\begin{equation}
    \label{eq:theo_1.1}
    \Delta_{k+1} = \norm{\tilde{\bm{\Sigma}}_{k+1} - \bm{\Sigma}^*_{k+1} - \bm{\mathcal{E}}_{1,k} - \bm{\mathcal{E}}_{2,k}}_F \leq \norm{\tilde{\bm{\Sigma}}_{k+1} - \bm{\Sigma}^*_{k+1}}_F + \norm{\bm{\mathcal{E}}_{1,k}}_F + \norm{\bm{\mathcal{E}}_{2,k}}_F
\end{equation}
Recall the definition of $\bm{\Sigma}_{k+1}^*$ from (\ref{eq:true-deflation-mats}). By Lemma~\ref{lem:true_deflate_eigen_decomp}, we have $\bm{\Sigma}^*_{k+1} = \bm{\Sigma}^*_k - \lambda_k^*\bfu_k^*\bfu_k^{*\top}$. Thus, the first term in (\ref{eq:theo_1.1}) can be bounded as
\begin{align*}
    \norm{\tilde{\bm{\Sigma}}_{k+1} - \bm{\Sigma}^*_{k+1}}_F & = \norm{\bm{\Sigma}_k - \lambda_k\bfu_{k}\bfu_{k}^\top - \bm{\Sigma}^*_k + \lambda_k^*\bfu_k^*\bfu_k^{*\top}}_F\\
    & \leq \norm{\bm{\Sigma}_k - \bm{\Sigma}^*_k}_F + \norm{\lambda_k\bfu_{k}\bfu_{k}^\top - \lambda_k^*\bfu_k^*\bfu_k^{*\top}}_F\\
    & \leq \Delta_k + \norm{\paren{\lambda_k - \lambda_k^*}\bfu_{k}\bfu_{k}^\top - \lambda_k^*\paren{\bfu_{k}^*\bfu_{k}^{* \top} -\bfu_k\bfu_k^\top}}_F\\
    & \leq 2\Delta_k + \lambda_k^*\norm{\bfu_{k}\bfu_{k}^\top -\bfu_k^*\bfu_k^{* \top}}_F
\end{align*}
We use Weyl's inequality in the last inequality to obtain that $\left|\lambda_k - \lambda_k^*\right|\leq \Delta_k$. The last term above can be further bounded as
\begin{align*}
    \norm{\bfu_{k}\bfu_{k}^\top -\bfu_k^*\bfu_k^{*\top}}_F  = \norm{\bfu_{k}\paren{\bfu_{k} - \bfu_k^*}^\top + \paren{\bfu_{k} - \bfu_k^*}\bfu_k^{*\top}}_F \leq 2\norm{\bfu_{k} - \bfu_k^*}_2 = 2\theta_k
\end{align*}
Therefore, we have
\begin{equation}
    \label{eq:theo_1.2}
    \norm{\tilde{\bm{\Sigma}}_{k+1} - \bm{\Sigma}^*_{k+1}}_F \leq 2\Delta_k + 2\lambda_k^*\theta_k
\end{equation}
Now, we focus on $\bm{\mathcal{E}}_{1,k}$. It can be bounded as
\[
    \norm{\bm{\mathcal{E}}_{1,k}}_F = \lambda_k\norm{\bm{\delta}_k\bfu_{k}^\top + \bfu_{k}\bm{\delta}_k^\top + \bm{\delta}_k\bm{\delta}_k^\top}_F \leq \lambda_k\paren{2\norm{\bm{\delta}_k}_2 + \norm{\bm{\delta}_k}_2^2}
\]
Given that $\lambda_k \leq \lambda_k^* + \left|\lambda_k - \lambda_k\right|\leq \lambda_k + \Delta_k$, we have
\begin{equation}
    \label{eq:theo_1.3}
    \norm{\bm{\mathcal{E}}_{1,k}}_F \leq \paren{\lambda_k^* + \Delta_k}\paren{2\norm{\bm{\delta}_k}_2 + \norm{\bm{\delta}_k}_2^2}
\end{equation}
Similarly, for $\bm{\mathcal{E}}_{2,k}$ we have
\begin{equation}
    \begin{aligned}
        \label{eq:theo_1.4}
        \norm{\bm{\mathcal{E}}_{2,k}}_F & \leq \left|2\lambda_k\bfu_{k}^\top\bm{\delta}_k + \bm{\delta}_k^\top\bm{\Sigma}_k\bm{\delta}_k\right|\norm{\bfu_{k} + \bm{\delta}_k}_2^2\\
        & \leq \paren{2\paren{\lambda_k^* + \Delta_k}\norm{\bm{\delta}_k}_2 + \paren{\lambda_k^* + \Delta_k}\norm{\bm{\delta}_k}_2^2}\paren{1 + \norm{\bm{\delta}_k}_2}^2\\
        & = \paren{\lambda_k^* + \Delta_k}\paren{2\norm{\bm{\delta}_k}_2 + \norm{\bm{\delta}_k}_2^2}\paren{1 + \norm{\bm{\delta}_k}_2}^2
    \end{aligned}
\end{equation}
where the first inequality follows from $\norm{\bm{\Sigma}_k}_2 \leq \lambda_k$ since, by definition, $\lambda_k$ is the largest eigenvalue of $\bm{\Sigma}_k$. Plugging (\ref{eq:theo_1.2}), (\ref{eq:theo_1.3}), and (\ref{eq:theo_1.4}) into (\ref{eq:theo_1.1}), we have
\begin{align*}
    \Delta_{k+1} &\leq2\Delta_k + 2\lambda_k^*\theta_k + \paren{\lambda_k^* + \Delta_k}\paren{2\norm{\bm{\delta}_k}_2 + \norm{\bm{\delta}_k}_2^2}\\
    & \quad\quad\quad +\paren{\lambda_k^* + \Delta_k}\paren{2\norm{\bm{\delta}_k}_2 + \norm{\bm{\delta}_k}_2^2}\paren{1 + \norm{\bm{\delta}_k}_2}^2\\
    & \leq 2\Delta_k + 2\lambda_k^*\theta_k + \paren{\lambda_k^* + \Delta_k}\paren{2\norm{\bm{\delta}_k}_2 + \norm{\bm{\delta}_k}_2^2}\paren{1 +\paren{1 + \norm{\bm{\delta}_k}_2}^2}\\
    & = 2\Delta_k + 2\lambda_k^*\theta_k + \paren{\lambda_k^* + \Delta_k}\norm{\bm{\delta}_k}_2\paren{2 + \norm{\bm{\delta}_k}_2}\paren{1 +\paren{1 + \norm{\bm{\delta}_k}_2}^2}
\end{align*}
When $\norm{\bm{\delta}_k}_2 \leq \frac{1}{6}$, we will have that
\[
    \paren{2 + \norm{\bm{\delta}_k}_2}\paren{1 +\paren{1 + \norm{\bm{\delta}_k}_2}^2} = \frac{13}{6}\cdot\paren{1 + \paren{\frac{7}{6}}^2} = \frac{1014}{216} \leq 5
\]
Thus, we have
\[
    \Delta_{k+1}\leq 2\Delta_k + 2\lambda_k^*\theta_k + 5\paren{\lambda_k^* + \Delta_k}\norm{\bm{\delta}_k}_2 \leq 3\Delta_k + 2\lambda_k^*\theta_k + 5\lambda_k^*\norm{\bm{\delta}_k}_2
\]

\subsection{Proof of Lemma~\ref{lem:top_eig_diff}}
To prove Lemma~\ref{lem:top_eig_diff}, we first recall the Davis-Kahan $\sin\Theta$ theorem \cite{davis1970rotation, eldridge2017unperturbed}.
\begin{theorem}[Davis-Kahan $\sin\Theta$ theorem]
    \label{theo:davis-kahan}
    Let $\bfM^*\in\R^{d\times d}$ and let $\bfM = \bfM^* + \bfH$. Let $\bfa_1^*$ be the top eigenvector of $\bfM^*$ and $\bfa_1$ be the top eigenvector of $\bfM$. Then we have
    \[
        \sin\angle\left\{\bfa_1^*,\bfa_1\right\} \leq \frac{\norm{\bfH}_2}{\min_{j\neq k}\left|\sigma_k^* - \sigma_j\right|}.
    \]
\end{theorem}
To start, we notice that, under the assumption that $\norm{\bm{\Sigma}_k - \bm{\Sigma}_k^*}_F \leq \frac{1}{4}\mathcal{T}_k$, by the Weyl's inequality, we have
\[
    \min_{j\neq k}\left|\sigma_k^* - \sigma_j\right| \geq \min_{j\neq k}\left|\sigma_k^* - \sigma_j^*\right| - \left|\sigma_j - \sigma_j^*\right| \geq \mathcal{T}_k - \norm{\bm{\Sigma}_k - \bm{\Sigma}_k^*}_F \geq \frac{3}{4}\mathcal{T}_k
\]
By Lemma~\ref{lem:true_deflate_eigen_decomp}, we have that $\bfu_k^*$ is the top eigenvector of $\bm{\Sigma}_k^*$ with corresponding eigenvector $\lambda_k^*$. Moreover, as defined in (\ref{eq:uk_defin}), $\bfu_k$ is the top eigenvector of $\bm{\Sigma}_k$. Therefore, a direct application of Theorem~\ref{theo:davis-kahan} gives
\begin{equation}
    \label{eq:lem3.5.1}
    \sin\angle\left\{\bfu_k,\bfu_k^*\right\} \leq \frac{4\norm{\bm{\Sigma}_k - \bm{\Sigma}_k^*}_F}{3\mathcal{T}_k}
\end{equation}
Moreover, notice that
\[
    \sin\angle\left\{\bfu_k,\bfu_k^*\right\}^2 = 1 - \cos\angle\left\{\bfu_k,\bfu_k^*\right\}^2 = 1- \paren{\bfu_k^\top\bfu_k^*}^2
\]
which gives that $\paren{\bfu_k^\top\bfu_k^*}^2 = 1 - \sin\angle\left\{\bfu_k,\bfu_k^*\right\}^2$. Therefore
\[  
    \norm{\bfu_k - \bfu_k^*}_2^2 = 2 - 2\bfu_k^\top\bfu_k^* \leq 2 - 2\paren{\bfu_k^\top\bfu_k^*}^2 = 2\sin\angle\left\{\bfu_k,\bfu_k^*\right\}^2
\]
Here the first inequality follows from the fact that $\left|\bfu_k^\top\bfu_k^*\right|\leq 1$. Applying (\ref{eq:lem3.5.1}) gives
\[
    \norm{\bfu_k - \bfu_k^*}_2 \leq \sqrt{2}\sin\angle\left\{\bfu_k,\bfu_k^*\right\} \leq \frac{4\sqrt{2}}{3\mathcal{T}_k}\norm{\bm{\Sigma}_k - \bm{\Sigma}_k^*}_F \leq \frac{2}{\mathcal{T}_k}\norm{\bm{\Sigma}_k - \bm{\Sigma}_k^*}_F
\]
\subsection{Proof of Theorem~\ref{thm:main_theorem_1}}
\label{sec:proof_main_theorem_1}
Let $\hat{K}$ be the smallest integer such that $\norm{\bm{\Sigma}_{\hat{K}} - \bm{\Sigma}_{\hat{K}}^*}_2 > \frac{1}{4}\mathcal{T}_{K,\min}$, and let $K'=\min\left\{\hat{K},K+1\right\}$. By definition, $\mathcal{T}_{K,\min} = \min_{k\in[K]}\mathcal{T}_k$. Therefore, for all $k < K'$, we shall have that $\norm{\bm{\Sigma}_k-\bm{\Sigma}_k^*}_2\leq \frac{1}{4}\mathcal{T}_k$. Therefore, for all $k\in[K']$, we can invoke Lemma~\ref{lem:top_eig_diff} to get that
\begin{equation}
    \label{eq:thm1.1}
    \norm{\bfu_k - \bfu_k^*}_2\leq \frac{2}{\mathcal{T}_k} \norm{\bm{\Sigma}_k-\bm{\Sigma}_k^*}_F
\end{equation}
By the assumption of Theorem~\ref{thm:main_theorem_1}, we have
\[
    \norm{\bm{\delta}_k}_2\leq \frac{\mathcal{T}_{K,\min}}{20K}\prod_{j=k+1}^{K-1}\paren{3 + \frac{2\lambda_j^*}{\mathcal{T}_j}}^{-1} \leq \frac{1}{6}
\]
since $\mathcal{T}_{K,\min}\leq \lambda_1^* = 1$ and $K\geq 1$. Therefore, we also have that $\norm{\bm{\delta}_k}_2\leq \frac{1}{6}$ for all $k\in[K'-1]$. This allows us to invoke Lemma~\ref{lem:matrix_diff_propagate} to get that
\begin{equation}
    \label{eq:thm1.2}
    \norm{\bm{\Sigma}_{k+1} - \bm{\Sigma}_{k+1}^*}_F\leq 3\norm{\bm{\Sigma}_k-\bm{\Sigma}_k^*}_F + 2\lambda_k^*\norm{\bfu_k - \bfu_k^*}_2 + 5\lambda_k^*\norm{\bm{\delta}_k}_2
\end{equation}
Plugging (\ref{eq:thm1.1}) into (\ref{eq:thm1.2}) gives
\[
    \norm{\bm{\Sigma}_{k+1} - \bm{\Sigma}_{k+1}^*}_F\leq \paren{3 + \frac{2\lambda_k^*}{\mathcal{T}_k}}\norm{\bm{\Sigma}_k - \bm{\Sigma}_k^*}_F + 5\lambda_k^*\norm{\bm{\delta}_k}_2
\]
Let the sequence $\left\{Q_k\right\}_{k=1}^{K'}$ be defined as
\[
    Q_{k+1} = a_kQ_k + b_k;\quad Q_0 = 0;\quad a_k = 3 + \frac{2\lambda_k^*}{\mathcal{T}_k};\quad b_k = 5\lambda_k^*\norm{\bm{\delta}_k}_2
\]
Then we must have that $\norm{\bm{\Sigma}_k - \bm{\Sigma}_k^*}_F \leq Q_k$ for all $k\in[K']$. Invoking Lemma~\ref{lem:prod_sum_seq} thus gives
\begin{equation}
    \label{eq:thm1.3}
    \norm{\bm{\Sigma}_k - \bm{\Sigma}_k^*}_F \leq 5\sum_{k'=1}^{k-1}\lambda_{k'}^*\norm{\bm{\delta}_{k'}}_2\prod_{j=k'+1}^{k-1}\paren{3 + \frac{2\lambda_j^*}{\mathcal{T}_j}}
\end{equation}
Combining with (\ref{eq:thm1.1}) and the fact that $\norm{\bfv_k - \bfu_k^*}_2 \leq \norm{\bm{\delta}_k}_2 + \norm{\bfu_k - \bfu_k^*}_2$ gives
\begin{equation}
    \label{eq:thm1.4}
    \norm{\bfv_k - \bfu_k^*}_2\leq \norm{\bm{\delta}_k}_2 + \frac{2}{\mathcal{T}_k}\norm{\bm{\Sigma}_k - \bm{\Sigma}_k^*}_F \leq \norm{\bm{\delta}_k}_2+\frac{10}{\mathcal{T}_k}\sum_{k'=1}^{k-1}\lambda_{k'}^*\norm{\bm{\delta}_{k'}}_2\prod_{j=k'+1}^{k-1}\paren{3 + \frac{2\lambda_j^*}{\mathcal{T}_j}}
\end{equation}
Notice that
\[
    \frac{10}{\mathcal{T}_k} \leq \frac{5}{\lambda_k^*}\cdot\paren{3+\frac{2\lambda_k^*}{\mathcal{T}_k}}
\]
Therefore, (\ref{eq:thm1.4}) becomes
\[
    \norm{\bfv_k - \bfu_k^*}_2 \leq \norm{\bm{\delta}_k}_2 +5\sum_{k'=1}^{k-1}\frac{\lambda_{k'}^*}{\lambda_k^*}\norm{\bm{\delta}_{k'}}_2\prod_{j=k'+1}^{k}\paren{3 + \frac{2\lambda_j^*}{\mathcal{T}_j}} \leq 5\sum_{k'=1}^{k}\frac{\lambda_{k'}^*}{\lambda_k^*}\norm{\bm{\delta}_{k'}}_2\prod_{j=k'+1}^{k}\paren{3 + \frac{2\lambda_j^*}{\mathcal{T}_j}}
\]
for all $k\in[K']$.
It thus remains to show that $\hat{K} > K$. In this way, we will have $K' = K$, and the theorem is then proved. For the sake of contradiction, that $\hat{K}\leq K$. Then we shall have that there exists $k\in[K]$ such that $\norm{\bm{\Sigma}_k - \bm{\Sigma}_k^*}_2 > \frac{1}{4}\mathcal{T}_{K,\min}$. To reach a contradiction, it thus remains to show that for all $k\in[K]$
\[
    \norm{\bm{\Sigma}_k - \bm{\Sigma}_k^*}_2 \leq \frac{1}{8}\mathcal{T}_{K,\min}
\]
By (\ref{eq:thm1.3}), it suffices to guarantee that
\begin{equation}
    \label{eq:thm1.5}
    5\sum_{k'=1}^{k-1}\lambda_{k'}^*\norm{\bm{\delta}_{k'}}_2\prod_{j=k'+1}^{k-1}\paren{3 + \frac{2\lambda_j^*}{\mathcal{T}_j}} \leq \frac{1}{4}\mathcal{T}_{K,\min}
\end{equation}
Indeed, by the assumption of Theorem~\ref{thm:main_theorem_1}, we require
\[
    \norm{\bm{\delta}_k}_2\leq \frac{\mathcal{T}_{K,\min}}{20K}\prod_{j=k+1}^{K-1}\paren{3 + \frac{2\lambda_j^*}{\mathcal{T}_j}}^{-1}
\]
Plugging this requirement into (\ref{eq:thm1.3}) would guarantee that $\norm{\bm{\Sigma}_k - \bm{\Sigma}_k^*}_2 \leq \frac{1}{4}\mathcal{T}_{K,\min}$ and thus achieves the contradiction. This shows that $\hat{K} > K$ and finishes the proof.

\section{Missing Proofs from Section~\ref{sec:main_result_pi}}
\label{sec:proof_main_result_pi}

\subsection{Proof of Lemma~\ref{lem:init_property}}
\label{sec:proof_init_property}
Since each $\bfx_{0,k}$ is generated independently, it suffices to show that for a general $\bfx_0$ generated for unit sphere and some unit vector $\bfu$, we have $\big|\bfx_0^\top\bfu|\geq $ with high probability, and then apply a union bound for all $k\in[K]$. To start, for a fixed $\bfu$, we let $\bfU$ be an orthogonal matrix with $\bfu$ being its first row. Since $\bfx_0$ is generated uniformly at random from a unit sphere, we must have that $\bfb = \bfU\bfx$ is also generated uniformly at random from a unit sphere. Let $b_j$ denote the $j$th entry of $\bfb$, then $b_1 = \bfx_0^\top\bfu$. We first aim at showing
\[
    \mathbb{P}\paren{\sum_{j=1}^d\paren{\frac{b_j}{b_1}}\leq 2d^3}\geq 1 - \frac{2}{3d}
\]
Let $\mathcal{B}_d(1)$ denote the unit ball in $\R^d$. Define the indicator function $\mathbb{I}\paren{\bfb}$ as
\[
    \mathbb{I}(\bfb) = \begin{cases}
        1 & \text{ if } \sum_{j=1}^d\paren{\frac{b_j}{b_1}}\leq 2d^3\\
        0 & \text{ otherwise}
    \end{cases}
\]
Then we have
\[
    \mathbb{P}\paren{\sum_{j=1}^d\paren{\frac{b_j}{b_1}}\leq 2d^3} = \int_{\norm{\bfx}_2 = 1}\mathbb{I}(\bfb)d\mu\paren{\bfb}
\]
where $\mu$ is the probability measure over the unit sphere. Since $\mathbb{I}\paren{\bfb}$ satisfies $\mathbb{I}\paren{\alpha\bfb} = \mathbb{I}\paren{\bfb}$ for any $\alpha\neq 0$ and $\mathbb{I}\paren{\bfb}$ does not depend on the signs of of $b_j$'s, then we can use Remark 7.2 from \citep{kuczynski1992estimating} to get that
\begin{align*}
    \int_{\norm{\bfx}_2 = 1}\mathbb{I}(\bfb)d\mu\paren{\bfb} & = \frac{1}{V_d}\int_{\mathcal{B}_d(1)}\mathbb{I}(\bfb)d\bfb \\
    & = 2\int_0^1\paren{\int_{\sum_{j=2}^db_j^2\leq \min\{1 - b_1^2,2b_1d^3\}}db_2\dots db_d}db_1\\
    & = \frac{2V_{d-1}}{V_d}\int_0^1\min\{1 - b_1^2,2b_1d^3\}^{\frac{d-1}{2}}db_1
\end{align*}
where the last equality follows from the observation that the inner integral is exactly the volume of the ball with radius $\min\{1 - b_1^2,2b_1d^3\}$. We can lower-bound the last integral as
\begin{align*}
    \int_0^1\min\{1 - b_1^2,2b_1d^3\}^{\frac{d-1}{2}}db_1 & \geq \int_{\sqrt{\frac{1}{1 + 2d^3}}}^1\paren{1 - b_1^2}^{\frac{d-1}{2}}db_1\\
    & = \int_0^1\paren{1 - b_1^2}^{\frac{d-1}{2}}db_1 - \int_0^{\sqrt{\frac{1}{1 + 2d^3}}}\paren{1 - b_1^2}^{\frac{d-1}{2}}db_1\\
    & \geq \frac{V_{d}}{2V_{d-1}} - \sqrt{\frac{1}{1 + 2d^3}}
\end{align*}
This gives
\[
    \mathbb{P}\paren{\sum_{j=1}^d\paren{\frac{b_j}{b_1}}\leq 2d^3}\geq 1 - \frac{2V_{d-1}}{V_d}\sqrt{\frac{1}{1 + 2d^3}}
\]
By \citep{kuczynski1992estimating}, Eq (13), we have $\frac{V_{d-1}}{V_d} \leq 0.412\sqrt{d}$ when $d\geq 8$. Thus, we have
\[
    \mathbb{P}\paren{\sum_{j=1}^d\paren{\frac{b_j}{b_1}}\leq 2d^3}\geq 1 - 0.824\sqrt{\frac{d}{1 + 2d^3}} \geq 1 - \frac{2}{3d}
\]
Notice that $\sum_{j=1}^db_j^2 = 1$. Then we have that 
\[
    \sum_{j=2}^d\paren{\frac{b_j}{b_1}}^2 = \frac{1}{b_1^2} - 1 \Rightarrow \big|b_1\big| = \paren{\sum_{j=2}^d\paren{\frac{b_j}{b_1}}^2 + 1}^{-\frac{1}{2}}
\]
Thus, we have
\[
    \mathbb{P}\paren{\big|b_1\big| \geq \frac{1}{\sqrt{1+ 2d^3}}}\geq 1 - \frac{2}{3d}
\]
Apply a union bound over all $k$ gives that with probability at least $1 - \frac{2K}{3d}$, we have
\[
    \big|\bfx_{0,k}^\top\bfu_k\big|\geq \frac{1}{\sqrt{1 + 2d^3}}
\]
\subsection{Proof of Lemma~\ref{lem:pw_mat_perturb}}
\label{sec:proof_pw_mat_perturb}
For convenience, we define the sequence $\left\{\hat{\bfx}_k\right\}_{k=0}^\infty$ as 
\[
    \hat{\bfx}_t = \bfM\hat{\bfx}_{t-1};\quad\hat{\bfx}_0 = \bfx_0
\]
Then $\bfx_t$ from (\ref{eq:power_iteration}) can be written as
\[
    \bfx_t = \frac{\hat{\bfx}_t}{\norm{\hat{\bfx}_t}_2} = \frac{\hat{\bfx}_t}{\norm{\bfM^t\bfx_0}_2}
\]
Focusing on $\hat{\bfx}_t$, since $\bfM = \bfM^* + \bfH$, we have
\[
    \hat{\bfx}_t = \bfM\hat{\bfx}_{t-1} = \bfM^*\hat{\bfx}_{t-1} +\bfH\hat{\bfx}_{t-1}
\]
Since $\sigma_j^*, \bfa_j^*$ are the $j$th eigenvalue and eigenvector of $\bfM^*$, we have that
\begin{align*}
    \hat{\bfx}_t^\top\bfa_j^* = \hat{\bfx}_{t-1}^\top\bfM^*\bfa_j^* + \hat{\bfx}_{t-1}^\top\bfH\bfa_j^* = \sigma_j^*\hat{\bfx}_{t-1}^\top\bfa_j^* + \hat{\bfx}_{t-1}^\top\bfH\bfa_j^*
\end{align*}
Unrolling the iterates gives
\[
    \hat{\bfx}_t^\top\bfa_j^* = \sigma_i^{*t}\bfx_0^\top\bfa_j^* + \sum_{t'=0}^{t-1}\sigma_j^{*t-t'}\hat{\bfx}_{t'}^\top\bfH\bfa_j^* = \sigma_j^{*t}\paren{\bfx_0^\top\bfa_j^* + \sum_{t'=0}^{t-1}\frac{1}{\sigma_j^{*t'}}\hat{\bfx}_{t'}^\top\bfH\bfa_j^*}
\]
Recalling that $\hat{\bfx}_t = \bfM^t\bfx_0$, we can see that
\[
    \hat{\bfx}_t^\top\bfa_j^* =\sigma_j^{*t}\paren{\bfx_0^\top\bfa_j^* + \sum_{t'=0}^{t-1}\paren{\frac{\bfM^{t'}}{\sigma_j^{*t'}}\hat{\bfx}_{0}}^\top\bfH\bfa_j^*} = \sigma_j^{*t}\bfx_0^\top\paren{\bfI + \sum_{t'=0}^{t-1}\frac{\bfM^{t'}}{\sigma_j^{*t'}}\cdot\bfH}\bfa_j^*
\]
Plugging in $\bfx_t = \frac{\hat{\bfx}_t}{\norm{\bfM^t\bfx_0}_2}$ gives the first conclusion. Next, without loss of generality, we assume that $\norm{\bfx_0}_2=1$, since $\left|\bfx_t^\top\bfa_j^*\right|$ does not depend on $\norm{\bfx_0}_2$. Since $\norm{\bfx_0}_2 = \norm{\bfa_j^*}_2 = 1$, it holds that
\begin{align*}
    \left|\hat{\bfx}_t^\top\bfa_j^*\right| & \leq \sigma_j^{*t}\paren{1 + \norm{\sum_{t'=0}^{t-1}\frac{\bfM^{t'}}{\sigma_j^{*t'}}\cdot\bfH\bfa_j^*}_2}\\
    & \leq \sigma_j^{*t}\paren{1 + \sum_{t'=0}^{t-1}\frac{\sigma_1^{t'}}{\sigma_j^{*t'}}\norm{\bfH\bfa_j^*}_2}\\
    & \leq \sigma_j^{*t}\paren{1 + \frac{\paren{\frac{\sigma_1}{\sigma_j^*}}^t - 1}{\frac{\sigma_1}{\sigma_j^*} - 1}\norm{\bfH\bfa_j^*}_2}\\
    & \leq \sigma_j^{*k} + \frac{\sigma_1^t - \sigma_i^{*t}}{\frac{\sigma_1}{\sigma_j^*} - 1}\norm{\bfH\bfa_j^*}_2\\
    & \leq \sigma_j^{*t} + \frac{\sigma_j^{*}\sigma_1^t}{\sigma_1-\sigma_j^*}\norm{\bfH\bfa_j^*}_2
\end{align*}
Since there exists some $\left|\bfx_0^\top\bfa_1\right|\geq c_0^{-1} > 0$, we can bound $\norm{\bfM^t\bfx_0}_2$ as
\[
    \norm{\bfM^t\bfx_0}_2 = \paren{\sum_{j=1}^r\sigma_j^{2t}\paren{\bfx_0^\top\bfa_j}^2}^{\frac{1}{2}} \geq \sigma_1^{t}\left|\bfx_0^\top\bfa_j\right| \geq \sigma_1^tc_0^{-1}
\]
Therefore, we have
\[
    \left|\bfx_t^\top\bfa_j^*\right|\leq \norm{\bfM^t\bfx_0}_2^{-1}\left|\hat{\bfx}_t^\top\bfa_j^*\right| \leq c_0\paren{\paren{\frac{\sigma_j^{*}}{\sigma_1}}^t + \frac{\sigma_j^{*}}{\sigma_1-\sigma_j^*}\norm{\bfH\bfa_j^*}_2}
\]
This shows the second conclusion.

\subsection{Proof of Lemma~\ref{lem:matrix_diff_align}}
\label{sec:proof_matrix_diff_align}
Recall the definition of the set of empirical deflated matrices
\[
    \bm{\Sigma}_{k+1} = \bm{\Sigma}_k - \bfv_k\bfv_k^\top\bm{\Sigma}_k\bfv_k\bfv_k^\top
\]
and the set of ground-truth deflated matrices
\[
    \bm{\Sigma}^*_{k+1} = \bm{\Sigma}^*_k - \bfu_k^*\bfu_k^{*\top}\bm{\Sigma}^*_k\bfu_k^*\bfu_k^{*\top}
\]
By Lemma~\ref{lem:true_deflate_eigen_decomp}, we have
\[
    \bm{\Sigma}^*_k = \sum_{j=k}^{n}\lambda_j^*\bfu_j^*\bfu_j^{*\top}
\]
Therefore, by the orthogonality between $\bfu_i^*$ and $\bfu_j^*$ when $i\neq j$, we have
\begin{equation}
    \label{eq:theo_2.1}
    \bm{\Sigma}^*_k\bfu_j^* = \sum_{i=k}^{n}\lambda_i^*\bfu_i^*\paren{\bfu_i^{*\top}\bfu_j^*} = \begin{cases}
    \lambda_j^*\bfu_j^* & \text{ if } j \geq k\\
    \bm{0} & \text{ if } j < k
    \end{cases}
\end{equation}
Next, we will focus on $\bm{\Sigma}_k\bfu_j^*$. Notice that
\[
    \bm{\Sigma}_{k+1}\bfu_j^* = \bm{\Sigma}_k\bfu_j^* - \bfv_k\bfv_k^\top\bm{\Sigma}_k\bfv_k\bfv_k^\top\bfu_j^*
\]
Therefore
\begin{equation}
    \label{eq:theo_2.2}
    \bm{\Sigma}_k\bfu_j^* = \bm{\Sigma}\bfu_j^* - \sum_{k'=1}^{k-1}\bfv_{k'}^\top\bfu_j^*\cdot\bfv_{k'}\bfv_{k'}^\top\bm{\Sigma}_{k'}\bfv_{k'} = \lambda_j^*\bfu_j^* -\sum_{k'=1}^{k-1}\bfv_{k'}^\top\bfu_j^*\cdot\bfv_{k'}^\top\bm{\Sigma}_{k'}\bfv_{k'}\cdot\bfv_{k'}
\end{equation}
Therefore, combining (\ref{eq:theo_2.1}) and (\ref{eq:theo_2.2}) gives
\[
    \paren{\bm{\Sigma}^*_k-\bm{\Sigma}_k}\bfu_j^* = \begin{cases}
    \sum_{k'=1}^{k-1}\bfv_{k'}^\top\bfu_j^*\cdot\bfv_{k'}^\top\bm{\Sigma}_{k'}\bfv_{k'}\cdot \bfv_{k'} & \text{ if } j \geq k\\
    \sum_{k'=1}^{k-1}\bfv_{k'}^\top\bfu_j^*\cdot\bfv_{k'}^\top\bm{\Sigma}_{k'}\bfv_{k'}\cdot \bfv_{k'} - \lambda_j^*\bfu_j & \text{ if } j < k
    \end{cases}
\]
This shows the first conclusion.
Next, we will put more emphasis on the case where $j\geq k$. Under this scenario, we can first show that
\[
    \bfv_{k'}^\top\bm{\Sigma}_{k'}\bfv_{k'} \leq \lambda_{k'}\norm{\bfv_{k'}}_2^2 = \lambda_{k'}
\]
Therefore,
\[
    \norm{\paren{\bm{\Sigma}^*_k-\bm{\Sigma}_k}\bfu_j^*}_2 
 \leq  \sum_{k'=1}^{k-1}\lambda_{k'}\left|\bfv_{k'}^\top\bfu_j^*\right|\cdot\norm{\bfv_k'}_2 = \sum_{k'=1}^{k-1}\lambda_{k'}\left|\bfv_{k'}^\top\bfu_j^*\right|
\]

\subsection{Proof of Lemma~\ref{lem:eigen_inner_equal}}
\label{sec:proof_eigen_inner_equal}

Let us consider $\bfa_i^\top\paren{\bfM^*+\bfH}\bfa_j^*$. On one hand, since $\bfa_i$ is an eigenvector of $\bfM^* + \bfH$, we have
\[
    \bfa_i^\top\paren{\bfM^*+\bfH}\bfa_j^* = \sigma_i\bfa_i^\top\bfa_j^*
\]
On the other hand, since $\bfa_i^*$ is the $i$th eigenvector of $\bfM^*$, we have
\[
    \bfa_i^\top\paren{\bfM^*+\bfH}\bfa_j^* = \sigma_j^*\bfa_i^\top\bfa_j^* + \bfa_i^\top\bfH\bfa_j^*
\]
Therefore,
\[
    \sigma_i\bfa_i^\top\bfa_j^* = \sigma_j^*\bfa_i^\top\bfa_j^* + \bfa_i^\top\bfH\bfa_j^*
\]
which implies that
\[
    \paren{\sigma_i - \sigma_j^*}\bfa_i^\top\bfa_j^* = \bfa_i^\top\bfH\bfa_j^*
\]
Thus, when $\sigma_i\neq \sigma_j^*$, we have
\[
    \bfa_i^\top\bfa_j^* = \frac{\bfa_i^\top\bfH\bfa_j^*}{\sigma_i - \sigma_j^*}
\]

\subsection{Proof of Lemma~\ref{lem:inner_ub_to_inner_lb}}
The proof of Lemma~\ref{lem:inner_ub_to_inner_lb} is based on the Neumann expansion \cite{eldridge2017unperturbed, chen2020asymmetry}, which we state below
\begin{lemma}[Theorem 7 of \cite{eldridge2017unperturbed}]
    \label{lem:neumann_expansion}
    Let $\bfM^*\in\R^{d\times d}$ be a rank-$r$ matrix, and let $\bfM = \bfM^* + \bfH$ for some $\bfH\in\R^{d\times d}$. Let $\sigma_i^*, \bfa_i^*$ be the $i$-th eigenvalue and eigenvector of $\bfM^*$, and $\sigma_i, \bfa_i$ be the $i$-th eigenvalue and eigenvector of $\bfM$. Suppose that $\norm{\bfH}_2\leq \left|\sigma_i\right|$. Then, 
    \[
        \bfa_i = \sum_{j=1}^r\frac{\sigma_j^*}{\sigma_i}\bfa_i^\top\bfa_j^*\paren{\sum_{s=0}^\infty\frac{1}{\sigma_i^s}\bfH^s\bfa_j^*}.
    \]
\end{lemma}
To start, using the Neumann expansion (Lemma~\ref{lem:neumann_expansion}), we prove an auxiliary lemma, as stated below.
\label{sec:inner_ub_to_inner_lb}
\begin{lemma}
    \label{lem:neumann_subspace_approx}
    Under the same setting of Lemma~\ref{lem:neumann_expansion}, we further assume that $\sigma_i^* \geq 0$ for all $i\in[r]$, and that $\norm{\bfH}_2\leq \frac{1}{8}\sigma_1^*$. Then we have
    \[
        \norm{\bfa_1 - \sum_{j=1}^r\frac{\sigma_j^*}{\sigma_1}\bfa_1^\top\bfa_j^*\cdot\bfa_j^*}_2\leq \frac{4}{3}\cdot\frac{\norm{\bfH}_2}{\sigma_1}
    \]
\end{lemma}
\begin{proof}
    Applying Lemma~\ref{lem:neumann_expansion} with $i=1$, we have
    \[
        \bfa_1 = \sum_{j=1}^r\frac{\sigma_j^*}{\sigma_1}\bfa_1^\top\bfa_j^*\paren{\bfa_j^* + \sum_{s=1}^\infty\frac{1}{\sigma_1^s}\bfH^s\bfa_j^*} = \sum_{j=1}^r\frac{\sigma_j^*}{\sigma_1}\bfa_1^\top\bfa_j^*\cdot\bfa_j^* + \sum_{j=1}^r\frac{\sigma_j^*}{\sigma_1}\bfa_1^\top\bfa_j^*\paren{\sum_{s=1}^\infty\frac{1}{\sigma_i^s}\bfH^s\bfa_j^*}
    \]
    Therefore, our quantity of interest can be written as
    \begin{equation}
    \label{eq:theo_4.1}
    \begin{aligned}
        \norm{\bfa_1 - \sum_{j=1}^r\frac{\sigma_j^*}{\sigma_1}\bfa_1^\top\bfa_j^*\cdot\bfa_j^*}_2 & = \norm{\sum_{j=1}^r\frac{\sigma_j^*}{\sigma_1}\bfa_1^\top\bfa_j^*\paren{\sum_{s=1}^\infty\frac{1}{\sigma_1^s}\bfH^s\bfa_j^*}}_2\\
        & = \norm{\sum_{s=1}^{\infty}\frac{\bfH^{s}}{\sigma_1^{s+1}}\sum_{j=1}^r\sigma_j^*\bfa_1^\top\bfa_j^*\cdot\bfa_j^*}_2\\
        & = \norm{\sum_{s=1}^{\infty}\frac{\bfH^{s}}{\sigma_1^{s+1}}\begin{bmatrix}\bfa_1^* & \dots & \bfa_r^*\end{bmatrix}\begin{bmatrix}\sigma_1^*\bfa_1^\top\bfa_1^*\\ \vdots\\ \sigma_r^*\bfa_1^\top\bfa_r^*\end{bmatrix}}_2\\
        & \leq \norm{\sum_{s=1}^{\infty}\frac{\bfH^{s}}{\sigma_1^{s+1}}}_2\cdot\norm{\begin{bmatrix}\bfa_1^* & \dots & \bfa_r^*\end{bmatrix}}_2\cdot\norm{\begin{bmatrix}\sigma_1^*\bfa_1^\top\bfa_1^*\\ \vdots\\ \sigma_r^*\bfa_1^\top\bfa_r^*\end{bmatrix}}_2
    \end{aligned}
    \end{equation}
    Now, we have a product of three terms to analyze. Recall that $\sigma_1^* > 0$. To start, by Weyl's inequality, we have $\sigma_1 \geq \sigma_1^* - \norm{\bfH}_2$. Since $\norm{\bfH}_2\leq\frac{1}{8}\sigma_1^*$, we must have that $\norm{\bfH}_2 \leq \sigma_1$. Applying the triangle inequality and the standard result of the geometric series gives
    \begin{equation}
        \label{eq:theo_4.2}
        \norm{\sum_{s=1}^{\infty}\frac{\bfH^{s}}{\sigma_1^{s+1}}}_2 \leq \sum_{s=1}^\infty\frac{\norm{\bfH}_2^s}{\sigma_1^{s+1}} = \frac{\norm{\bfH}_2}{\sigma_1\paren{\sigma_1 - \norm{\bfH}_2}}
    \end{equation}
    Next, since $\bfa_1^*,\dots,\bfa_r^*$ are orthonormal vectors, we must have that
    \begin{equation}
        \label{eq:theo_4.3}
        \norm{\begin{bmatrix}\bfa_1^* & \dots & \bfa_r^*\end{bmatrix}}_2 \leq 1
    \end{equation}
    For the same reason, and since $\sigma_i^*\geq 0$ for all $i\in[r]$, we also have that
    \begin{equation}
        \label{eq:theo_4.4}
        \norm{\begin{bmatrix}\sigma_1^*\bfa_1^\top\bfa_1^*\\ \vdots\\ \sigma_r^*\bfa_1^\top\bfa_r^*\end{bmatrix}}_2 \leq \sigma_1^*\norm{\begin{bmatrix}\bfa_1^\top\bfa_1^*\\ \vdots\\ \bfa_1^\top\bfa_r^*\end{bmatrix}}_2 = \sigma_1^*\norm{\bfa_1^\top\begin{bmatrix}\bfa_1^*,\dots,\bfa_r^*\end{bmatrix}}_2 \leq \sigma_1^*
    \end{equation}
    To put things together, we plug (\ref{eq:theo_4.2}), (\ref{eq:theo_4.3}), and (\ref{eq:theo_4.4}) into (\ref{eq:theo_4.1}) to get that
    \begin{equation}
        \label{eq:theo_4.5}
        \norm{\bfa_1 - \sum_{j=1}^r\frac{\sigma_j^*}{\sigma_1}\bfa_1^\top\bfa_j^*\cdot\bfa_j^*}_2\leq \frac{\norm{\bfH}_2}{\sigma_1\paren{\sigma_1 - \norm{\bfH}_2}}\cdot 1\cdot\sigma_1^* = \frac{\sigma_1^*\norm{\bfH}_2}{\sigma_1\paren{\sigma_1 - \norm{\bfH}_2}}
    \end{equation}
    Apply again $\norm{\bfH}_2\leq \frac{1}{8}\sigma_1^*$, we have
    \[
        \sigma_1-\norm{\bfH}_2 \geq \sigma_1^* - 2\norm{\bfH}_2 \geq \frac{3}{4}\sigma_1^{*}
    \]
    Therefore, (\ref{eq:theo_4.5}) becomes
    \begin{equation}
        \label{eq:theo_4.6}
        \norm{\bfa_1 - \sum_{j=1}^r\frac{\sigma_j^*}{\sigma_1}\bfa_1^\top\bfa_j^*\cdot\bfa_j^*}_2\leq \frac{4}{3}\cdot\frac{\norm{\bfH}_2}{\sigma_1}
    \end{equation}
\end{proof}
Now, we will use Lemma~\ref{lem:neumann_subspace_approx} to prove the desired statements.
Notice that the Euclidean projection of $\bfa_1$ onto the subspace spanned by $\left\{\bfa_j^*\right\}_{j=1}^r$ (namely $\mathcal{S}^*$) is given by $\sum_{j=1}^r\bfa_1^\top\bfa_j^*\cdot\bfa_j^*$. Therefore, we have
\begin{equation}
    \label{eq:lem7.1}
    \begin{aligned}
        \paren{1 - \sum_{j=1}^r\paren{\bfa_1^\top\bfa_j^*}^2}^{\frac{1}{2}} & = \norm{\bfa_1 - \sum_{j=1}^r\bfa_1^\top\bfa_j^*\cdot\bfa_j^*}_2\\
        & = \min_{\bfz\in\mathcal{S}^*}\norm{\bfa_1 - \bfz}_2\\
        & \leq \norm{\bfa_1 - \sum_{j=1}^r\frac{\sigma_j^*}{\sigma_1}\bfa_1^\top\bfa_j^*\cdot\bfa_j^*}_2
    \end{aligned}
\end{equation}
where the inequality follows from the fact that $\sum_{j=1}^r\frac{\sigma_j^*}{\sigma_1}\bfa_1^\top\bfa_j^*\cdot\bfa_j^*\in\mathcal{S}^*$. With the assumption of the lemma, we can invoke Lemma~\ref{lem:neumann_subspace_approx} to get that 
\[
    \norm{\bfa_1 - \sum_{j=1}^r\frac{\sigma_j^*}{\sigma_1}\bfa_1^\top\bfa_j^*\cdot\bfa_j^*}_2\leq \frac{4}{3}\cdot\frac{\norm{\bfH}_2}{\sigma_1} \leq \frac{32}{21}\cdot\frac{\norm{\bfH}_2}{\sigma_1^*}
\]
Combining with (\ref{eq:lem7.1}) gives
\[
    \sum_{j=1}^r\paren{\bfa_1^\top\bfa_j^*}^2 \geq 1 - \paren{\frac{32}{21}\cdot\frac{\norm{\bfH}_2}{\sigma_1^*}}^2 \geq 1 - 2.4\cdot  \frac{\norm{\bfH}_2^2}{\sigma_1^{*2}}
\]
Moving $\sum_{j=2}^r\paren{\bfa_1^\top\bfa_j^*}^2$ to the right-hand side of the inequality gives
\[
    \paren{\bfa_1^\top\bfa_1^*}^2  \geq 1 - 2.4\cdot  \frac{\norm{\bfH}_2^2}{\sigma_1^{*2}} - \sum_{j=2}^r\paren{\bfa_1^\top\bfa_j^*}^2
\]
The proof then conclude by letting $\bfa_1 = \bfu_k,\sigma_1^* = \lambda_k^*, \bfa_j^* = \bfu_{k+j-1}^*$ for all $j\in[d-k+1]$, and $\bfH = \bm{\Sigma}_k - \bm{\Sigma}_k^*$.

\subsection{Proof of Lemma~\ref{lem:pw_top_eig_diff}}
\label{sec:proof_pw_top_eig_diff}
To start, by the definition of $c_0$, we invoke Lemma~\ref{lem:pw_mat_perturb} with $\bfM = \bm{\Sigma}_k, \bfM^*=\bm{\Sigma}_k^*$ to obtain that
\begin{equation}
    \label{eq:lem14.1}
    \left|\bfv_k^\top\bfu_j^*\right|\leq c_0\paren{\paren{\frac{\lambda_j^*}{\lambda_k}}^t + \frac{\lambda_j^*}{\lambda_k - \lambda_i^*}\norm{\paren{\bm{\Sigma}_k - \bm{\Sigma}_k^*}\bfu_j^*}_2}
\end{equation}
Also, invoking Lemma~\ref{lem:matrix_diff_align}, we have that when $j\geq k$,  
\begin{equation}
    \label{eq:lem14.2}
    \norm{\paren{\bm{\Sigma}_k - \bm{\Sigma}_k^*}\bfu_j^*}_2 \leq \sum_{k'=1}^{k-1}\lambda_{k'}\left|\bfv_{k'}^\top\bfu_j^*\right|
\end{equation}
Plugging the upper bound of $\norm{\paren{\bm{\Sigma}_k - \bm{\Sigma}_k^*}\bfu_j^*}_2$ from (\ref{eq:lem14.2}) into (\ref{eq:lem14.1}), we have that when $j\geq k$
\[
    \left|\bfv_k^\top\bfu_j^*\right| \leq c_0\paren{\paren{\frac{\lambda_j^*}{\lambda_k}}^t + \frac{\lambda_j^*}{\lambda_k - \lambda_j^*}\sum_{k'=1}^{k-1}\lambda_{k'}\left|\bfv_{k'}^\top\bfu_j^*\right|}
\]
Equivalently, we can multiply both sides of the inequality with $\lambda_k$ to obtain that
\begin{equation}
    \label{eq:lem14.3}
    \lambda_k\left|\bfv_k^\top\bfu_j^*\right| \leq \lambda_kc_0\paren{\paren{\frac{\lambda_j^*}{\lambda_k}}^t + \frac{\lambda_j^*}{\lambda_k - \lambda_j^*}\sum_{k'=1}^{k-1}\lambda_{k'}\left|\bfv_{k'}^\top\bfu_j^*\right|}
\end{equation}
Notice that, the form of the inequality in (\ref{eq:lem14.3}) allows us to invoke Lemma~\ref{lem:lem:sum_acc_seq_growth_exact} to analyze the growth of $\lambda_k\left|\bfv_k^\top\bfu_j^*\right|$. In particular, setting $a_k = c_0\lambda_k\paren{\frac{\lambda_j^*}{\lambda_k}}^t$ for $k > 1$ and $a_1 = \lambda_1\left|\bfv_1^\top\bfu_j^*\right|$, and $b_k = \frac{c_0\lambda_k\lambda_j^*}{\lambda_k - \lambda_j^*}$, we can have that $\lambda_k\left|\bfv_k^\top\bfu_j^*\right| \leq Q_k$ for all $k$ for $Q_k$ defined in Lemma~\ref{lem:lem:sum_acc_seq_growth_exact}. Therefore, Lemma~\ref{lem:lem:sum_acc_seq_growth_exact} implies that
\begin{equation}
    \label{eq:lem14.4}
    \sum_{k'=1}^{k-1}\lambda_{k'}\left|\bfv_{k'}^\top\bfu_j^*\right| \leq c_0\sum_{k'=1}^{k-1}\lambda_{k'}\paren{\frac{\lambda_j^*}{\lambda_k}}^t\prod_{\ell=k'+1}^{k-1}\paren{1 + \frac{c_0\lambda_{\ell}\lambda_j^*}{\lambda_{\ell} - \lambda_j^*}}
\end{equation}
Now, we focus on the term $\paren{1 + \frac{c_0\lambda_{\ell}\lambda_j^*}{\lambda_{\ell} - \lambda_j^*}}$ in (\ref{eq:lem14.4}). In particular, we should notice that $\paren{1 + \frac{c_0\lambda_{\ell}\lambda_j^*}{\lambda_{\ell} - \lambda_j^*}}$ is only used in (\ref{eq:lem14.4}) for $\ell < k$. We can observe that
\[
    \frac{c_0\lambda_{\ell}\lambda_j^*}{\lambda_{\ell} - \lambda_j^*} = c_0\paren{\frac{1}{\lambda_j^*} - \frac{1}{\lambda_{\ell}}}^{-1}
\]
This implies that $\paren{1 + \frac{c_0\lambda_{\ell}\lambda_j^*}{\lambda_{\ell} - \lambda_j^*}}$ increases monotonically as $\lambda_{\ell}$ decrease. Recall that $\left|\lambda_{\ell}^* - \lambda_{\ell}\right|\leq \norm{\bm{\Sigma}_k - \bm{\Sigma}_k^*}_F \leq \frac{1}{8}\mathcal{T}_{k,\min}$. This implies that for all $\ell < k$, we must have that $\lambda_{\ell}\geq \lambda_{\ell+1}^*\geq \lambda_{k+1}^* \geq \lambda_k^*$. Thus, we must have that
\begin{equation}
    \label{eq:lem14.5}
    1 + \frac{c_0\lambda_{\ell}\lambda_j^*}{\lambda_{\ell} - \lambda_j^*}  \leq 1 + \frac{c_0\lambda_k^*\lambda_j^*}{\lambda_k^*-\lambda_j^*} \leq 1 + \frac{c_0\lambda_k^*\lambda_{k+1}^*}{\lambda_k^* - \lambda_{k+1}^*}
\end{equation}
Recall that in the statement of Lemma~\ref{lem:pw_top_eig_diff} we defined $\mathcal{G}_k =  1 + \frac{c_0\lambda_k^*\lambda_{k+1}^*}{\lambda_k^* - \lambda_{k+1}^*}$. Then (\ref{eq:lem14.5}) implies that (\ref{eq:lem14.4}) can be simplified Suppose that $\lambda_k \geq \lambda_j^*$, then we have
\begin{equation}
    \label{eq:lem14.6}
    \sum_{k'=1}^{k-1}\lambda_{k'}\left|\bfv_{k'}^\top\bfu_j^*\right| \leq c_0\sum_{k'=1}^{k-1}\lambda_{k'}\mathcal{G}_k^{k-k'-1}\paren{\frac{\lambda_j^*}{\lambda_k}}^{t}\leq c_0\sum_{k'=1}^{k-1}\mathcal{G}_k^{k-k'-1}\paren{\frac{\lambda_j^*}{\lambda_k}}^{t}
\end{equation}
where the last inequality follows from the fact that
\[
    \lambda_{k'}\leq \lambda_{k'}^* +\norm{\bm{\Sigma}_{k'} - \bm{\Sigma}_{k'}^*}_2 \leq \lambda_2^* + \frac{1}{8}\mathcal{T}_{k',\min}\leq  \lambda_1^* = 1
\]
for all $k \geq 2$, and $\lambda_1 = \lambda_1^* = 1$ by definition. Plugging (\ref{eq:lem14.6}) into the bound on $\norm{\paren{\bm{\Sigma}_k - \bm{\Sigma}_k^*}\bfu_j^*}_2$ in (\ref{eq:lem14.2}) gives
\begin{equation}
    \label{eq:lem14.7}
    \norm{\paren{\bm{\Sigma}_k - \bm{\Sigma}_k^*}\bfu_j^*}_2\leq c_0\sum_{k'=1}^{k-1}\mathcal{G}_k^{k-k'-1}\paren{\frac{\lambda_j^*}{\lambda_k}}^{t}
\end{equation}
Now, we can invoke Lemma~\ref{lem:eigen_inner_equal} with $\bfM = \bm{\Sigma}_k,\bfM^*=\bm{\Sigma}_k^*$, and correspondingly $\bfa_1 = \bfu_k, \bfa_j = \bfu_j^*$ to get that
\begin{equation}
    \label{eq:lem14.8}
    \left|\bfu_k^\top\bfu_j^*\right| = \left|\frac{\bfu_k^\top\paren{\bm{\Sigma}_k - \bm{\Sigma}_k^*}\bfu_j^*}{\lambda_k - \lambda_j^*}\right|\leq \frac{\norm{\paren{\bm{\Sigma}_k - \bm{\Sigma}_k^*}\bfu_j^*}_2}{\left|\lambda_k - \lambda_j^*\right|}
\end{equation}
This allows us to plug (\ref{eq:lem14.7}) into (\ref{eq:lem14.8}) to obtain that
\[
    \left|\bfu_k^\top\bfu_j^*\right| \leq \frac{c_0}{\left|\lambda_k - \lambda_j^*\right|}\sum_{k'=1}^{k-1}\mathcal{G}_k^{k-k'-1}\paren{\frac{\lambda_j^*}{\lambda_k}}^{t}
\]
Since $\norm{\bm{\Sigma}_{k'} - \bm{\Sigma}_{k'}^*}_F \leq \frac{1}{8}\mathcal{T}_{k,\min}$ for all $k'\in[k-1]$, then we have that $\lambda_{k-1}\geq \lambda_{k-1}^*-\frac{1}{8}\mathcal{T}_{k-1} \geq \lambda_{k}^*+\frac{1}{8}\mathcal{T}_{k-1} \geq \lambda_k$. Apply Lemma~\ref{lem:sum_geo_pow}, we have that as long as $t \geq \frac{\log \mathcal{G}_k}{\log \lambda_{k'} - \lambda_{k'+1}}$ for all $k'\in[k-1]$, 
\begin{align*}
    \left|\bfu_k^\top\bfu_j^*\right| & \leq \frac{c_0}{\left|\lambda_k - \lambda_j^*\right|}\sum_{k'=1}^{k-1}\mathcal{G}_k^{k-k'-1}\paren{\frac{\lambda_j^*}{\lambda_k}}^{t}\\
    & = \frac{c_0\lambda_j^{*t}\mathcal{G}_k^{k-1}}{\left|\lambda_k - \lambda_j^*\right|}\sum_{k'=1}^{k-1}\frac{\paren{\lambda_{k'}^{-1}}^t}{\mathcal{G}_k^{k'}}\\
    & \leq \frac{c_0\lambda_j^{*t}\mathcal{G}^{k-1}}{\left|\lambda_k - \lambda_j^*\right|}\cdot\paren{1 - \mathcal{G}_k\cdot\max_{k'\in[k-1]}\paren{\frac{\lambda_{k'+1}}{\lambda_{k'}}}^{t}}^{-1}\frac{\lambda_{k-1}^{-t}}{\mathcal{G}_k^{k-1}}\\
    & = \frac{c_0}{\left|\lambda_k - \lambda_j^*\right|}\paren{1 - \mathcal{G}_k\cdot\max_{k'\in[k-1]}\paren{\frac{\lambda_{k'+1}}{\lambda_{k'}}}^{t}}^{-1}\paren{\frac{\lambda_j^*}{\lambda_{k-1}}}^{t}
\end{align*}
In fact, as long as $t \geq \frac{\log 2\mathcal{G}_k}{\log \lambda_{k'} - \lambda_{k'+1}}$, we must have that for all $k'\in[k-1]$,
\[
    \mathcal{G}_k\cdot\max_{k'\in[k-1]}\paren{\frac{\lambda_{k'+1}}{\lambda_{k'}}}^{t} \leq \frac{1}{2}
\]
which implies that for all $j > k$
\begin{equation}
    \label{eq:lem14.9}
    \left|\bfu_k^\top\bfu_j^*\right| \leq\frac{2c_0}{\left|\lambda_k - \lambda_j^*\right|}\paren{\frac{\lambda_j^*}{\lambda_{k-1}}}^{t} \leq \frac{16c_0}{7\mathcal{T}_k}\paren{\frac{\lambda_j^*}{\lambda_{k}^*}}^{t}
\end{equation}
where the last inequality follows from 
\[
    \left|\lambda_k - \lambda_j^*\right| \geq \left|\lambda_k^* - \frac{1}{8}\mathcal{T}_k - \lambda_j^*\right| \geq \frac{7}{8}\mathcal{T}_k 
\]
since $\lambda_k \geq \lambda_k^* - \norm{\bm{\Sigma}_k - \bm{\Sigma}_k^*}_2 \geq \lambda_k^* - \frac{1}{8}\mathcal{T}_k$, and $\lambda_k^* \geq \lambda_j^* + \mathcal{T}_k$ for $j > k$.
Applying Lemma~\ref{lem:inner_ub_to_inner_lb}, we have that
\[
    \paren{\bfu_k^\top\bfu_k^*}^2 \geq 1 -\frac{2.4}{\lambda_k^{*2}}\norm{\bm{\Sigma}_k-\bm{\Sigma}_k^*}_2^2 - \sum_{j=k+1}^d\paren{\bfu_k^\top\bfu_j^*}^2
\]
Therefore, $\norm{\bfu_k - \bfu_k^*}_2^2$ can be bounded as
\begin{equation}
    \label{eq:lem14.10}
    \norm{\bfu_k - \bfu_k^*}_2^2 \leq 2 \paren{1 - \paren{\bfu_k^\top\bfu_k^*}^2} \leq \frac{4.8}{\lambda_k^*}\norm{\bm{\Sigma}_k-\bm{\Sigma}_k^*}_2^2 + \sum_{j=k+1}^d\paren{\bfu_k^\top\bfu_j^*}^2
\end{equation}
We finally plug (\ref{eq:lem14.9}) into (\ref{eq:lem14.10}) to get
\[
    \norm{\bfu_k - \bfu_k^*}_2^2\leq \frac{4.8}{\lambda_k^*}\norm{\bm{\Sigma}_k-\bm{\Sigma}_k^*}_2^2 + \frac{11c_0^2}{2\mathcal{T}_k^2}\sum_{j=k+1}^d\paren{\frac{\lambda_j^*}{\lambda_{k}^*}}^{2t}
\]
as desired. Lastly, we should notice that since $\norm{\bm{\Sigma}_k - \bm{\Sigma}_k^*}_F \leq \frac{1}{8}\mathcal{T}_{k,\min} \leq \lambda_{k'}^* - \lambda_{k'+1}^*$ for all $k'\in[k]$, it must hold that $\lambda_{k'} \geq \lambda_{k'}^* - \norm{\bm{\Sigma}_k - \bm{\Sigma}_k^*}_F \geq \frac{1}{8}\lambda_{k'+1}^* + \frac{7}{8}\lambda_{k'}^*$ and $\lambda_{k'+1} \leq \lambda_{k'+1}^* + \norm{\bm{\Sigma}_k - \bm{\Sigma}_k^*}_F\leq  \frac{1}{8}\lambda_{k'}^* + \frac{7}{8}\lambda_{k'+1}^*$ for all $k'\in[k-1]$. Thus, the condition that $t \geq \frac{\log 2\mathcal{G}_k}{\log \lambda_{k'} - \lambda_{k'+1}}$ can be satisfied as long as
\[
    t \geq \log2\mathcal{G}_k\paren{\log \frac{\frac{1}{8}\lambda_{k'+1}^* + \frac{7}{8}\lambda_{k'}^*}{\frac{1}{8}\lambda_{k'}^* + \frac{7}{8}\lambda_{k'+1}^*}}^{-1} = \log2\mathcal{G}_k\paren{\log \frac{\lambda_{k'+1}^* + 7\lambda_{k'}^*}{\lambda_{k'}^* + 7\lambda_{k'+1}^*}}^{-1}
\]
which is provided as an assumption in the lemma. This then finishes the proof.

\subsection{Proof of Theorem~\ref{thm:power-iteration}}
\label{sec:proof_power_iteration}
Let $\hat{K}$ be the smallest integer such that $\norm{\bm{\Sigma}_{\hat{K}} - \bm{\Sigma}_{\hat{K}}^*}_2 > \frac{1}{8}\mathcal{T}_{K,\min}$, and let $K'=\min\left\{\hat{K},K+1\right\}$. Therefore, for all $k < K'$, we shall have that $\norm{\bm{\Sigma}_k-\bm{\Sigma}_k^*}_2\leq \frac{1}{8}\mathcal{T}_{K,\min}$. Moreover, since $t \geq \log2\mathcal{G}_k\paren{\log \frac{\lambda_{k'+1}^* + 7\lambda_{k'}^*}{\lambda_{k'}^* + 7\lambda_{k'+1}^*}}^{-1}$, we can invoke Lemma~\ref{lem:pw_top_eig_diff} to have that 
\begin{equation}
    \label{eq:thm9.1}
    \norm{\bfu_k - \bfu_k^*}_2^2\leq \frac{4.8}{\lambda_k^*}\norm{\bm{\Sigma}_k-\bm{\Sigma}_k^*}_2^2 + \frac{11c_0^2}{2\mathcal{T}_k^2}\sum_{j=k+1}^d\paren{\frac{\lambda_j^*}{\lambda_{k}^*}}^{2t}
\end{equation}
Applying Lemma~\ref{lem:sum_geo_pow} again with $g = 1$ and $p_{k'} = \lambda_{d-k' +1}^*$ yields that, as long as $t \geq \frac{1}{\log \lambda_{j}^* - \lambda_{j+1}^*}$ for all $j > k$, it holds that
\[
    \sum_{j = k+1}^d\lambda_j^{*2t} \leq \paren{1 - \max_{j > k}\paren{\frac{\lambda_j^*}{\lambda_{j+1}^*}}^{2t}} \lambda_{k+1}^{*2t} \leq 2\lambda_{k+1}^{*2t}
\]
This transforms (\ref{eq:thm9.1}) into
\begin{equation}
    \label{eq:thm9.2}
    \norm{\bfu_k - \bfu_k^*}_2^2\leq\frac{4.8}{\lambda_k^*}\norm{\bm{\Sigma}_k-\bm{\Sigma}_k^*}_2^2 + \frac{11c_0^2}{\mathcal{T}_k^2}\paren{\frac{\lambda_{k+1}^*}{\lambda_k^*}}^{2t}
\end{equation}
Now, we invoke Lemma~\ref{lem:matrix_diff_propagate} to obtain that, as long as $\norm{\bm{\delta}_k}\leq \frac{1}{6}$, we shall have that
\begin{equation}
    \label{eq:thm9.3}
    \norm{\bm{\Sigma}_{k+1} - \bm{\Sigma}_{k+1}^*}_F \leq 3\norm{\bm{\Sigma}_k - \bm{\Sigma}_k^*}_F + 2\lambda_k^*\norm{\bfu_k - \bfu_k^*}_2 + 5\lambda_k^*\norm{\bm{\delta}_k}_2
\end{equation}
We first plug (\ref{eq:thm9.2}) into (\ref{eq:thm9.3}) by applying $\sqrt{a+b}\leq \sqrt{a} + \sqrt{b}$ to have that
\begin{equation}
    \label{eq:thm9.4}
    \norm{\bm{\Sigma}_{k+1} - \bm{\Sigma}_{k+1}^*}_F \leq 8\norm{\bm{\Sigma}_k - \bm{\Sigma}_k^*}_F + \lambda_k^*\paren{5\norm{\bm{\delta}_k}_2 + \frac{7c_0}{\mathcal{T}_{k}}\paren{\frac{\lambda_{k+1}^*}{\lambda_k^*}}^t}
\end{equation}
Now, we can apply Lemma~\ref{lem:prod_sum_seq} to unroll (\ref{eq:thm9.4}) to get that
\begin{equation}
    \label{eq:thm9.5}
    \norm{\bm{\Sigma}_k - \bm{\Sigma}_k^*}_F \leq \sum_{k'=1}^{k-1}8^{k-k'}\lambda_{k'}^*\paren{5\norm{\bm{\delta}_{k'}}_2 + \frac{7c_0}{\mathcal{T}_{k'}}\paren{\frac{\lambda_{k'+1}^*}{\lambda_{k'}^*}}^t}
\end{equation}
Next, we can plug (\ref{eq:thm9.5}) into (\ref{eq:thm9.2}) to get that
\begin{equation}
    \label{eq:thm9.6}
    \norm{\bfu_k - \bfu_k^*}_2 \leq 3\sum_{k'=1}^{k-1}8^{k-k'}\frac{\lambda_{k'}^*}{\lambda_k^*}\paren{5\norm{\bm{\delta}_{k'}}_2 + \frac{7c_0}{\mathcal{T}_{k'}}\paren{\frac{\lambda_{k'+1}^*}{\lambda_{k'}^*}}^t} + \frac{4c_0}{\mathcal{T}_k}\paren{\frac{\lambda_{k+1}^*}{\lambda_k^*}}^t
\end{equation}
Finally, recall the definition of $\bfv_k = \bfu_k + \bm{\delta}_k$. Therefore, we have
\begin{align*}
    \norm{\bfv_k - \bfu_k^*}_2 & \leq 3\sum_{k'=1}^{k-1}8^{k-k'}\frac{\lambda_{k'}^*}{\lambda_k^*}\paren{5\norm{\bm{\delta}_{k'}}_2 + \frac{7c_0}{\mathcal{T}_{k'}}\paren{\frac{\lambda_{k'+1}^*}{\lambda_{k'}^*}}^t} + \frac{4c_0}{\mathcal{T}_k}\paren{\frac{\lambda_{k+1}^*}{\lambda_k^*}}^t + \norm{\bm{\delta}_k}_2\\
    & \leq3\sum_{k'=1}^{k}8^{k-k'}\frac{\lambda_{k'}^*}{\lambda_k^*}\paren{5\norm{\bm{\delta}_{k'}}_2 + \frac{7c_0}{\mathcal{T}_{k'}}\paren{\frac{\lambda_{k'+1}^*}{\lambda_{k'}^*}}^t}
\end{align*}
for all $k\in[K']$.
It thus remains to show that $\hat{K} > K$. In this way, we will have $K' = K$ and the theorem is then proved. Assume, for the sake of contradiction, that $\hat{K}\leq K$. Then we shall have that there exists $k\in[K]$ such that $\norm{\bm{\Sigma}_k - \bm{\Sigma}_k^*}_2 > \frac{1}{8}\mathcal{T}_{K,\min}$. To reach a contradiction, it thus remains to show that for all $k\in[K]$
\[
    \norm{\bm{\Sigma}_k - \bm{\Sigma}_k^*}_2 \leq \frac{1}{8}\mathcal{T}_{K,\min}
\]
By (\ref{eq:thm9.5}), it suffices to guarantee that
\begin{equation}
    \label{eq:thm9.7}
    \sum_{k'=1}^{K-1}8^{K-k'}\lambda_{k'}^*\paren{5\norm{\bm{\delta}_{k'}}_2 + \frac{7c_0}{\mathcal{T}_{k'}}\paren{\frac{\lambda_{k'+1}^*}{\lambda_{k'}^*}}^t} \leq \frac{1}{8}\mathcal{T}_{K,\min}
\end{equation}
Moreover, recall that $\bfv_k$ is the output of the $k$th power iteration procedure. Therefore, we can invoke Lemma~\ref{lem:pw_err_eig_align} to have that
\[
    \norm{\bm{\delta}_k}_2 \leq 2\paren{\frac{\lambda_{2}\paren{\bm{\Sigma}_k}}{\lambda_k}}^t
\]
Let $r_k = \max\left\{\frac{\lambda_{2}\paren{\bm{\Sigma}_k}}{\lambda_k}, \frac{\lambda_{k+1}^*}{\lambda_k^*}\right\}$. Therefore, 
\[
    5\norm{\bm{\delta}_{k'}}_2 + \frac{7c_0}{\mathcal{T}_{k'}}\paren{\frac{\lambda_{k'+1}^*}{\lambda_{k'}^*}}^t \leq \frac{17c_0}{\mathcal{T}_{k'}}r_{k'}^t
\]
Thus, to guarantee (\ref{eq:thm9.7}), it suffices to guarantee that
\[
    \sum_{k'=1}^{K-1}8^{K-k'}\frac{\lambda_{k'}^*}{\mathcal{T}_{k'}}r_{k'}^t \leq \frac{\mathcal{T}_{K,\min}}{140c_0}
\]
which, since $\lambda_k\leq \lambda_1 = \lambda_1^* = 1$, is satisfied as long as 
\[
    t \geq \frac{3(K-k) + \log \frac{140c_0K}{\mathcal{T}_{K,\min}^2}}{\log r_{k}^{-1}}
\]
Since $\norm{\bm{\Sigma}_k - \bm{\Sigma}_k^*}_F\leq \frac{1}{8}\mathcal{T}_{k,\min}$, we have that $r_k \leq \frac{7\lambda_{k+1}^* + \lambda_k}{7\lambda_k^* + \lambda_{k+1}^*}$.  Therefore, the requirement on $t$ becomes
\[
    t \geq \paren{\log\frac{7\lambda_{k}^* + \lambda_{k+1}^*}{7\lambda_{k+1}^* + \lambda_{k}^*}}^{-1}\paren{3(K-k) + \log \frac{140c_0K}{\mathcal{T}_{K,\min}^2}}
\]
for all $k\in[K]$, which is provided as an assumption of the Theorem. This finishes the proof.

\section{Auxiliary Lemmas}
\label{sec:auxiliary_lem}
\begin{lemma}
    \label{lem:true_deflate_eigen_decomp}
    Let $\left\{\bm{\Sigma}_k^*\right\}_{k=1}^d$ be defined in (\ref{eq:true-deflation-mats}). For all $k\in[d]$, we have
    \[
        \bm{\Sigma}^*_k = \sum_{j=k}^d\lambda_j^*\bfu_j^*\bfu_j^{*\top};\quad \bm{\Sigma}_k^*\bfu_{\ell}^* = \begin{cases}
            \lambda_{\ell}^*\bfu_{\ell}^* & \text{if } {\ell}\geq k\\
            \bm{0} & \text{if } {\ell} < k
        \end{cases}
    \]
\end{lemma}
\begin{proof}
    We shall prove the statement by induction. For the base case, when $k = 1$, we have $\bm{\Sigma}^*_k = \bm{\Sigma}$, and, by the form of the eigen-decomposition, we have
    \[
        \bm{\Sigma}^*_k = \sum_{j=1}^d\lambda_j^*\bfu_j^*\bfu_j^{*\top};\quad \bfu_k^*\bfu_k^{*\top}\bm{\Sigma}^*_k\bfu_k^*\bfu_k^{*\top} = \lambda_k^*\bfu_k^*\bfu_k^{*\top}
    \]
    Now, suppose that we have
    \[
        \bm{\Sigma}^*_{k} = \sum_{j=k}^d\lambda_j^*\bfu_j^*\bfu_j^{*\top}
    \]
    We can see that 
    \[
        \bfu_k^*\bfu_k^{*\top}\bm{\Sigma}^*_{k}\bfu_k^*\bfu_k^{*\top} = \sum_{j=k}^d\lambda_j^*\paren{\bfu_k^{*\top}\bfu_j^*}^2 \bfu_k^*\bfu_k^{*\top} = \lambda_k^*\bfu_k^*\bfu_k^{*\top}
    \]
    where the last equality follows from the orthogonality between $\bfu_k^*$ and $\bfu_j^*$ when $j\neq k$. Therefore
    \[
        \bm{\Sigma}^*_{k+1} = \sum_{j=k}^d\lambda_j^*\bfu_j^*\bfu_j^{*\top} - \lambda_k^*\bfu_k^*\bfu_k^{*\top} = \sum_{j=k+1}^d\lambda_j^*\bfu_j^*\bfu_j^{*\top}
    \]
    which shows the inductive step and thus proves the first statement. To see the second statement, we notice that $\bfu_j^{*\top}\bfu_{\ell}^* = 0$ for all $j\neq \ell$ and $\bfu_j^{*\top}\bfu_{\ell}^* = 0$ if $j = \ell$. Therefore
    \[
        \bm{\Sigma}_k^*\bfu_{\ell}^* = \sum_{j=k}^d\bfu_j^{*\top}\bfu_{\ell}^*\cdot\lambda_j^*\bfu_j^* = \begin{cases}
            \lambda_{\ell}^*\bfu_{\ell}^* & \text{if } {\ell}\geq k\\
            \bm{0} & \text{if } {\ell} < k
        \end{cases}
    \]
\end{proof}

\begin{lemma}
    \label{lem:pw_convergence}
    Consider the power iteration procedure in (\ref{eq:power_iteration}). Let $\left\{\sigma_j\right\}_{j=1}^r$ be the eigenvalues of $\bfM$ in (\ref{eq:power_iteration}). Then we have
    \[
        \min_{s\in\{\pm 1\}}\norm{s\cdot\bfx_t - \bfa_1}_2^2 \leq 2\paren{\frac{\sigma_2}{\sigma_1}}^t
    \]
\end{lemma}
\begin{proof}
    For convenience, we define the sequence $\left\{\hat{\bfx}_k\right\}_{k=0}^\infty$ as 
    \[
        \hat{\bfx}_t = \bfM\hat{\bfx}_{t-1};\quad\hat{\bfx}_0 = \bfx_0
    \]
    Then $\bfx_t$ from (\ref{eq:power_iteration}) can be written as
    \[
        \bfx_t = \frac{\hat{\bfx}_t}{\norm{\hat{\bfx}_t}_2} = \frac{\hat{\bfx}_t}{\norm{\bfM^t\bfx_0}_2}
    \]
    Let $\{\bfa_j\}_{j=1}^d$ be an extended orthogonal basis of $\{\bfa_j\}_{j=1}^r$. Then we can write
    \[
        \bfx_0 = \sum_{j=1}^d\bfx_0^\top\bfa_j\cdot \bfa_j
    \]
    Let $\sigma_j = 0$ for $j\geq r+1$. Then we have that $\bfM\bfa_j = \sigma_j\bfa_j$. Therefore
    \[
        \bfM^t\bfx_0 = \sum_{j=1}^d\bfx_0^\top\bfa_j\cdot\bfM^t\bfa_j = \sum_{j=1}^d\bfx_0^\top\bfa_j\cdot \sigma_j^{t}\bfa_j = \sum_{j=1}^r\bfx_0^\top\bfa_j\cdot\sigma_j^{t}\bfa_j
    \]
    Since $\norm{\bfM^t\bfx_0}_2\leq\sigma_1^t$, we must have that
    \[
        \left|\bfx_t^\top\bfa_j\right| = \frac{\left|\hat{\bfx}_t\top\bfa_j\right|}{\norm{\bfM^t\bfx_0}_2} \leq \paren{\frac{\sigma_j}{\sigma_1}}^t\left|\bfx_0^\top\bfa_j\right|
    \]
    Since $\left\{\bfa_j\right\}_{j=1}^d$ is an orthogonal basis of $\R^d$, we must have that $\norm{\bfx_t}_2^2 = \sum_{j=1}^d\paren{\bfx_t^\top\bfa_j}^2$. Thus
    \[
        \paren{\bfx_t^\top\bfa_1}^2 = 1 -\sum_{j=2}^d\paren{\bfx_t^\top\bfa_j}^2 \geq 1 - \sum_{j=2}^d\paren{\frac{\sigma_j}{\sigma_1}}^{2t}\paren{\bfx_0^\top\bfa_j}^2 \geq 1 - \paren{\frac{\sigma_2}{\sigma_1}}^{2t}
    \]
    where the last inequality follows from $\sigma_j\leq \sigma_2$ for $j\geq 2$ and $\sum_{j=1}^d\paren{\bfx_0^\top\bfa_j}^2 = \norm{\bfx_0}_2^2 = 1$. Therefore, we can conclude that
    \[  
        \min_{s\in\{\pm 1\}}\norm{s\cdot\bfx_t - \bfa_1}_2^2\leq 2\paren{1 - \paren{\bfx_t^\top\bfa_1}^2}\leq 2\paren{\frac{\sigma_2}{\sigma_1}}^{2t}
    \]
    Taking a square root for both sides gives the desired result.
\end{proof}

\begin{lemma}
    \label{lem:lem:sum_acc_seq_growth_exact}
    Consider a sequence of quantities $\left\{Q_k\right\}_{k=1}^\infty$ satisfying
    \[
        Q_k = a_k + b_k\sum_{k'=1}^{k-1}Q_{k'}
    \]
    with some $a_k, b_k\geq 0$ for all $k\in\Z_+$. Setting $Q_1 = a_1$, then we shall have that
    \[
        Q_k = a_k + b_k\sum_{k'=1}^{k-1}a_{k'}\prod_{\ell=k'+1}^{k-1}\paren{1 + b_{\ell}}
    \]
\end{lemma}
\begin{proof}
    We shall prove by induction start from $Q_1 = a_1$. Suppose that $Q_1,\dots Q_k$ satisfies the characterization. To start, we can see that
    \begin{align*}
        \sum_{k'=1}^{k-1}Q_{k'} = \frac{1}{b_k}\paren{Q_k - a_k} = \sum_{k'=1}^{k-1}a_{k'}\prod_{\ell=k'+1}^{k-1}\paren{1 + b_{\ell}}
    \end{align*}
    Then, by the inductive hypothesis, we have
    \begin{align*}
        \sum_{k'=1}^kQ_{k'} & = Q_k + \sum_{k'=1}^{k-1}Q_{k'}\\
        & = a_k + \paren{b_k+1}\sum_{k'=1}^{k-1}a_{k'}\prod_{\ell=k'+1}^{k-1}\paren{1 + b_{\ell}}\\
        & = a_k + \sum_{k'=1}^{k-1}a_{k'}\prod_{\ell=k'+1}^{k}\paren{1 + b_{\ell}}\\
        & = \sum_{k'=1}^ka_{k'}\prod_{\ell=k'+1}^{k}\paren{1 + b_{\ell}}
    \end{align*}
    This implies that
    \[
        Q_{k+1} = a_{k+1} + b_{k+1}\sum_{k'=1}^{k}Q_{k'} = a_{k+1} + b_{k+1}\sum_{k'=1}^ka_{k'}\prod_{\ell=k'+1}^{k}\paren{1 + b_{\ell}}
    \]
    which finishes the proof.
\end{proof}

\begin{lemma}
    \label{lem:prod_sum_seq}
    Consider a sequence of quantities $\left\{Q_k\right\}_{k=1}^\infty$ satisfying
    \[
        Q_{k+1} = a_kQ_k + b_k
    \]
    with some $a_k, b_k \geq 0$ for all $k\in\Z_+$. Set $b_0 = Q_1$. Then we have that
    \[
        Q_k = \sum_{k'=0}^{k-1}b_{k'}\prod_{j=k'+1}^{k-1}a_j
    \]
\end{lemma}
\begin{proof}
    We shall prove by induction. For the base case, let $k= 1$. In this case, we have that
    \[
        Q_1 = \sum_{k'=0}^0b_{k'}\prod_{j=k'+1}^{0}a_j = b_0 = Q_1
    \]
    For the inductive case, assume that the property holds for $k$. Then we have that
    \[
        Q_{k+1} = a_kQ_k+ b_k = a_k\cdot \sum_{k'=0}^{k-1}b_{k'}\prod_{j=k'+1}^{k-1}a_j + b_k = \sum_{k'=0}^{k}b_{k'}\prod_{j=k'+1}^{k}a_j
    \]
    This proves the inductive step and finishes the proof.
\end{proof}

\begin{lemma}
    \label{lem:sum_geo_pow}
    Let $\left\{p_k\right\}_{k=1}^\infty$ be a positive increasing sequence, and $g, t > 0$. Let the series $S_k$ be defined as
    \[
        S_k = \sum_{k'=1}^{k-1}\frac{p_{k'}^t}{g^{k'}}
    \]
    As long as $t > \frac{\log g}{\log p_{k'} - \log p_{k'+1}}$ for all $k'\in[k-1]$, we shall have that
    \[
        S_k \leq \paren{1 - \max_{k'\in[k-1]}g\paren{\frac{p_{k'}}{p_{k'+1}}}^t}^{-1}\frac{p_{k-1}^t}{g^{k-1}}
    \]
\end{lemma}
\begin{proof}
    Set $\gamma = \max_{k'\in[k-1]}g\paren{\frac{p_{k'}}{p_{k'+1}}}^t$. Since $t > \frac{\log g}{\log p_{k'+1} - \log p_{k'}}$ for all $k'\in[k-1]$, we must have that $\gamma < 1$. Moreover, by definition of $\gamma$, it must holds that for all $k'\in[k-1]$
    \[
        \frac{p_{k'}^t}{g^{k'}}\leq \gamma\cdot \frac{p_{k'+1}^t}{g^{k'+1}}
    \]
    Thus, we have that for all $k'\in[k-1]$
    \[
        \frac{p_{k'}^t}{g^{k'}}\leq \gamma^{k-k'-1}\cdot \frac{p_{k-1}^t}{g^{k-1}}
    \]
    Therefore $S_k$ can be upper-bounded as
    \[
        S_k\leq \sum_{k'=1}^{k-1}\gamma^{k-k'-1}\cdot \frac{p_{k-1}^t}{g^{k-1}} = \frac{p_{k-1}^t}{g^{k-1}}\sum_{k'=0}^{k-2}\gamma^{k'} \leq \frac{p_{k-1}^t}{g^{k-1}(1-\gamma)}
    \]
\end{proof}

\begin{lemma}   
    \label{lem:pw_err_eig_align}
    Let $\bfx_t$ be the result of running power iteration starting from $\bfx_0$ for $t$ iterations, as defined in (\ref{eq:power_iteration}). Let $\sigma_j,\bfa_j$ be the $j$-th eigenvalue and eigenvector of $\bfM$ in (\ref{eq:power_iteration}). Assume that there exists $c_0 > 0$ such that $\left|\bfx_0^\top\bfa_1\right|\geq c_0^{-1}$. Then for all $j=2,\dots,r$, we have:
    \[
        \left|\bfx_t^\top\bfa_j\right| = \left|\paren{\bfx_t - \bfa_1}^\top\bfa_j\right| \leq c_0\paren{\frac{\sigma_j}{\sigma_1}}^t.
    \]
\end{lemma}
\begin{proof}
    We notice that for $j\geq 2$, it holds that $\paren{\bfx_t - \bfa_1}^\top\bfa_j = \bfx_t^\top\bfa_j$.
    Since $\bfM$ is symmetric, we have that
    \[
        \bfx_t^\top\bfa_j = \alpha\bfx_0^\top\bfM^t\bfa_j = \alpha\sigma_j^{t}\bfx_0^\top\bfa_j = \frac{\sigma_j^{t}}{\norm{\bfM^t\bfx_0}_2}\bfx_0^\top\bfa_j
    \]
    This brings the focus to $\norm{\bfM^t\bfx_0}_2$. Let $r$ be the rank of $\bfM$. Let $\{\bfa_j\}_{j=1}^d$ be an extended orthogonal basis of $\{\bfa_j\}_{j=1}^r$. Then we can write
    \[
        \bfx_0 = \sum_{j=1}^d\bfx_0^\top\bfa_j\cdot \bfa_j
    \]
    Let $\sigma_j = 0$ for $j\geq r+1$. Then we have that $\bfM\bfa_j = \sigma_j\bfa_j$. Therefore
    \[
        \bfM^t\bfx_0 = \sum_{j=1}^d\bfx_0^\top\bfa_j\cdot\bfM^t\bfa_j = \sum_{j=1}^d\bfx_0^\top\bfa_j\cdot \sigma_j^{t}\bfa_j = \sum_{j=1}^r\bfx_0^\top\bfa_j\cdot\sigma_j^{*t}\bfa_j
    \]
    where the last equality follows from $\sigma_j^{t} = 0$ when $t > 0$. Therefore, according to the Pythagorean Theorem,
    \[
        \norm{\bfM^k\bfx_0}_2 = \paren{\sum_{j=1}^r\sigma_j^{2t}\paren{\bfx_0^\top\bfa_j}^2}^{\frac{1}{2}} \geq \sigma_1^{t}\left|\bfx_0^\top\bfa_j\right|
    \]
    Recall that $\left|\bfx_0^\top\bfa_j\right|\geq c_0^{-1} > 0$. Thus, we have
    \[
        \left|\paren{\bfx_t - \bfa_1}^\top\bfa_j\right| = \frac{\sigma_j^{t}}{\norm{\bfM^t\bfx_0}_2}\left|\bfx_0^\top\bfa_j\right| \leq \paren{\frac{\sigma_j}{\sigma_1}}^t\cdot\left|\frac{\bfx_0^\top\bfa_j}{\bfx_0^\top\bfa_1}\right| \leq \left|\bfx_0^\top\bfa_1\right|^{-1}\paren{\frac{\sigma_j}{\sigma_1}}^t \leq c_0\paren{\frac{\sigma_j}{\sigma_1}}^t
    \]
\end{proof}

\section{Experiment Details of Figure~\ref{fig:dfl_spec_clst}}
\label{sec:exp_details}
We conduct spectral clustering \citep{belkin2003eigenmap} on a subset of the MNIST dataset \citep{deng2012mnist}, namely $\mathcal{S} = \left\{\bfx_i\right\}_{i=1}^n$. In particular, our experiment starts with building a weighted graph $\mathcal{G} = \paren{V, E}$ where each node in the graph $\mathcal{G}$ represents a sample image, and each edge weight is chosen according to
\[
    e_{ij} = \begin{cases}
        \exp\paren{-\frac{1}{2}\norm{\bfx_i-\bfx_j}_2^2} & \text{ if } \bfx_i\in\texttt{rNN}\paren{\bfx_j} \text{ or } \bfx_j\in\texttt{rNN}\paren{\bfx_i}\\
        0 & \text{ otherwise}
    \end{cases}
\]
Here $\texttt{rNN}\paren{\bfx_i}$ represents the set of $r$ Nearest Neighbors of $\bfx_i$. To be more specific, two nodes in the graph are considered adjacent if the two corresponding samples are top-$r$ nearest neighbors, and the weight of the edge between the two nodes is computed using the RBF kernel with unit variance. Next, we compute the degree-normalized graph Laplacian as our similarity matrix
\[
    \bfL = \bfI - \bfD^{-\frac{1}{2}}\bfA\bfD^{-\frac{1}{2}}
\]
where $\bfD\in\R^{n\times n}$ is the diagonal degree matrix, and $\bfA\in\R^{n\times n}$ is the adjacency matrix. We extract the top-$k$ eigenvectors of the similarity matrix $\bfL$ with the deflation algorithm in Algorithm~\ref{alg:main-alg} with $t$ power iteration steps in each call to the $\texttt{PCA}$ sub-routine, which results in the approximate eigenvectors $\bfu_1,\dots,\bfu_k$. We form a matrix $\bfU =[\bfu_1, \dots ,\bfu_k]\in\R^{n\times k}$, and assign the $i$th sample with a new feature vector $\bfv_i$ as the $i$th row of $\bfU$. Lastly, we perform $k'$-means clustering with the newly extracted features $\left\{\bfv_i\right\}_{i=1}^n$, forming clusters $C = \left\{C_j\right\}_{j=1}^{k'}$, where $C_j \subseteq [n]$.
The clustering result is evaluated using the mutual information metric \citep{Hubert1985ComparingP}.
\begin{equation}
    \label{eq:mutual_info}
    \text{MI}\paren{C, C^*} = \sum_{j,j'=1}^{k'}\mathbb{P}\paren{C_j\cap C^*_{j'}}\log\frac{\mathbb{P}\paren{C_j\cap C_{j'}^*}}{\mathbb{P}\paren{C_j}\mathbb{P}\paren{C_{j'}^*}}
\end{equation}
where $\mathbb{P}\paren{C_j} = \frac{|C_j|}{n}$ and ,similarly, $\mathbb{P}\paren{C_j\cap C_{j'}^*} = \frac{\left|C_j\cap C_j^*\right|}{n}$. Here $C^* = \left\{C_j^*\right\}_{j=1}^{k'}$ is the ground-truth clustering with $C_j^* = \{i\in[n]: \text{the label of sample $i$ is $j$}\}$.

In our experiment, we use a subset of MNIST with $n=1000$ randomly drawn samples. When constructing the graph, we use $r=10$, that is, we connect each node to ten of its nearest neighbors. We cluster the samples into $k'=10$ clusters. In Figure~\ref{fig:dfl_spec_clst}, the x-axis represents the number of power iteration steps $t$ in each call to the sub-routine $\texttt{PCA}$, and the y-axis represents the mutual information score in (\ref{eq:mutual_info}). Each line in the figure represents different $k$, which is the number of eigenvectors extracted from $\bfL$.

\end{document}